\documentclass[accepted]{uai2026} 
                        
\usepackage[american]{babel}

\usepackage{natbib} 
\usepackage{mathtools} 
\usepackage{booktabs} 
\usepackage{tikz} 

\usepackage{microtype}
\usepackage{graphicx}
\usepackage{array}
\usepackage{float}
\usepackage{subcaption}   
\usepackage{hyperref}

\usepackage{amsmath}
\usepackage{amssymb}
\usepackage{amsthm}

\usepackage[capitalize,noabbrev]{cleveref}

\DeclareMathOperator*{\argmin}{arg\,min}
\def\Var{{\rm Var}}
\def\E{{\mathbb{E}}}

\theoremstyle{plain}
\newtheorem{theorem}{Theorem}
\newtheorem*{theorem*}{Theorem}
\newtheorem{proposition}{Proposition}
\newtheorem*{proposition*}{Proposition}
\newtheorem{lemma}{Lemma}
\newtheorem{corollary}{Corollary}
\newtheorem*{corollary*}{Corollary}
\newtheorem{assumption}{Assumption}
\newtheorem*{assumption*}{Assumption}
\theoremstyle{definition}

\theoremstyle{example}
\newtheorem{example}{Example}
\theoremstyle{remark}
\newtheorem{remark}{Remark}

\usepackage[textsize=tiny]{todonotes}

\title{On Pairwise Quantile Regression - Statistical Guarantees and Applications}

\author[1]{\href{mailto:<romain.therezien@telecom-paris.fr>?Subject=Your UAI 2026 paper}{Romain Therezien}}
\author[1]{\href{mailto:<stephan.clemencon@telecom-paris.fr>?Subject=Your UAI 2026 paper}{Stephan Clémençon}}
\author[2]{\href{mailto:<fantin.girard@idemia.com>?Subject=Your UAI 2026 paper}{ Fantin Girard}}
\author[3]{\href{mailto:<hamza.el_abdouni@telecom-sudparis.eu>?Subject=Your UAI 2026 paper}{Hamza El-Abdouni}}
\affil[1]{%
    LTCI, Télécom Paris \\
    Institut Polytechnique de Paris \\
}
\affil[2]{%
    IDEMIA
}

\affil[3]{%
    Télécom SudParis, Institut Polytechnique de Paris
}

\begin{document}
\maketitle
\begin{abstract}
Quantile regression provides a powerful tool for summarizing the conditional distribution 
of a real-valued random variable (r.v.) of interest $Y$ as a function of covariates $Z$ in cases where it shows a large dispersion with high probability, going beyond the situation where standard least square regression is informative/predictive. This article aims to extend this methodology to the pairwise setting, where the variable to be explained is a similarity score between two independent observations (e.g., pixelated ID photos used as input to biometric systems), and the explanatory variables consist of the pair of covariates attached to these observations, such as age or hair color. We establish theoretical guarantees for solutions of this statistical learning problem, considered here as empirical minimizers of a pairwise version of the pinball loss. Leveraging sharp concentration results for $U$-processes, we prove generalization bounds and identify mild conditions under which fast learning rates can be achieved. Confirming the probabilistic analysis, experiments based on simulation data also provide solid empirical evidence of the validity of the methodology promoted here for pairwise quantile regression. Finally, its usefulness from an application perspective is demonstrated by a detailed study aimed at analyzing errors in similarity scoring for facial recognition.
\end{abstract}

\section{Introduction}

In many cases, the distribution of a variable to be explained or predicted, conditionally upon its covariates, is too spread out to be effectively summarized by its expectation. Traditional least squares regression is therefore of little interest, and \textit{quantile regression} proves to be an effective approach for describing the effect of explanatory variables on the distribution of the response variable. Rather than calculating the quantile at a given level $\tau\in (0,1)$ of a conditional distribution estimate (necessarily confronted with the curse of dimensionality) in a plug-in manner, the framework originally developed in \cite{koenker1978} allows the problem to be formulated according to the principle of empirical risk minimization (ERM). Considered thus as an M-estimation problem based on the asymmetric pinball loss, variants of numerous algorithms, \textit{e.g.} random forests, neural nets or (non linear) SVM, initially designed to deal with classification or least squares regression problems, have been proposed in the literature to recover conditional quantiles in a flexible and robust way, see \citep{10.1007/978-1-4612-2856-1_25, schoelkopf2000, 10.5555/1248547.1248582, JMLR:v7:takeuchi06a} for instance. Their performance is supported by theoretical guarantees in certain cases, refer to \citep{Steinwart_2011, JMLR:v23:21-0309}. The purpose of this article is to extend the quantile regression methodology and the statistical guarantees that underpin its validity, to the case where the output variable is a pairwise function $s(X,X')$ of two independent observations and the input takes the form of a pair of variables $Z$ and $Z'$, respective covariates of $X$ and $X'$, based on training examples $(X_1,Z_1),\; \ldots,\; (X_n,Z_n)$, independent copies of the generic pair $(X,Z)$. The analysis of pairwise data is of high importance in many modern machine learning applications related to clustering, ranking or metric-learning, see \citep{CL11, JMLR:v17:14-265, BHS15}, including drug discovery, information retrieval, recommending systems or biometric verification to name just a few. In this perspective, pairwise quantile regression is formulated here as the problem of minimizing a $U$-statistic, namely the empirical pinball loss averaged over all pairs $\left((X_i,Z_i),(X_j,Z_j)\right)$, over a class $\mathcal{F}$ of quantile regression function candidates. While minimization of $U$-statistics has been studied for pairwise learning problems such as ranking \citep{clemencon_empirical}, clustering \citep{CLEMENCON201442} and similarity scoring \citep{VBC18}, we prove here that applying this approach to pairwise quantile regression yields predictive functions with generalization guarantees, achieving fast learning rates of order $O_{\mathbb{P}}(1/n)$ under mild conditions on the conditional distribution of $s(X,X')$ given $(Z,Z')$. Beyond the theoretical analysis, we address the difficulties inherent in the dependence of the pairs of observations used. We quantify the benefit of using all the $n(n-1)/2$ pairs for the empirical pairwise pinball loss computation in terms of learning rate. We present numerical experiments demonstrating that the pairwise quantile regression methodology works very well in practice for appropriate model choices and allows us to understand the effect of interactions between the covariates $Z$ and $Z'$ on the variable $s(X,X')$ of interest. In addition, we propose a real world application of Facial Recognition (FR) systems \citep{Grother2019}, to analyze the distributions of similarity scores between photos around specific quantile levels, based on the properties of the images being compared (\textit{e.g.} disparities between resolution levels) and of the individuals appearing in them. As we show, the use of natively explainable models (\textit{e.g.} linear, random forests) or appropriate post-processing methods \citep{lundberg_unified_2017}  makes it possible to understand the impact of the properties of the images/identities being compared on the errors made by the FR system at supposedly high confidence levels. 

The article is organized as follows. In Section \ref{sec:background}, the key concepts of (pointwise) quantile regression are briefly recalled, along with elements related to statistical pairwise learning, in particular the theory of $U$-statistics/processes.  The notion of similarity scores, particularly in the context of FR systems, is also reviewed therein. Section \ref{sec:main} extends the statistical guarantees for quantile regression to the pairwise setting, providing generalization bounds. The advantage provided by the $U$-statistic form of the empirical pairwise loss is demonstrated, enabling fast learning rates under low noise conditions. In Section \ref{sec:exp}, our theoretical results are corroborated by numerical experiments on synthetic data, and the practical value of the methodology analyzed here is illustrated through an application based on a real dataset in the field of FR system performance analysis. Technical details and additional experiments are deferred to the Supplementary Material.



\section{Background and Preliminaries}\label{sec:background}
This section reviews the theoretical concepts underlying the approach subsequently analyzed, mainly (pointwise) quantile regression via empirical pinball loss minimization and the theory of $U$-statistics with an emphasis on the relevance for FR applications.

\subsection{(Pointwise) Quantile Regression}\label{subsec:point_QR}
\label{section:quantile}

First introduced in \cite{koenker1978}, quantile regression is an $M$-estimation technique for estimating the conditional quantiles of a real-valued response variable $Y \in \mathbb{R}$ given a set of covariates $Z\in\mathcal{Z}\subset\mathbb{R}^d$ with $d\ge 1$. Unlike least-squares regression methods, which aim to statistically recover the conditional mean $\mathbb{E}[Y \mid Z]$, it offers a more comprehensive understanding of the distributional effects of the covariates across quantile levels, thus providing more information about the conditional distribution of $Y$ given $Z$. For a fixed quantile level $\tau \in (0,1)$, the $\tau$-th conditional quantile of $Y$ given $Z$ is defined as
\begin{equation}\label{eq:conditionalQuantile}
Q_Y(\tau \mid Z) := \inf \{ y \in \mathbb{R} : F_{Y \mid Z}(y \mid Z) \ge \tau \},
\end{equation}
where by $F_{Y \mid Z}(\cdot \mid Z)$ is meant the conditional cumulative distribution function of $Y$ given $Z$. As in many other statistical learning problems (\textit{e.g.} classification, regression), the decision function \eqref{eq:conditionalQuantile}, which separates the proportion $1-\tau$ of the largest observations $Y$ from the remaining fraction $\tau$ given $Z$, can be formulated as a risk minimizer. Precisely, $z\in \mathcal{Z}\mapsto Q_Y(\tau \mid z)$ minimizes the risk
\begin{equation}\label{eq:riskPinballLoss}
    \mathcal{R}(f) := \mathbb{E}\left[ \rho_{\tau}(Y-f(Z))  \right],
\end{equation}
over the class of functions $f:\mathcal{Z}\to \mathbb{R}$ s.t. the expectation \eqref{eq:riskPinballLoss} is well-defined, meaning by $\rho_\tau(u) := u(\tau - \mathbb{I}\{u < 0\})$ the \textit{pinball loss}, where $\mathbb{I}\{\mathcal{E}\}$ denotes the indicator function of any event $\mathcal{E}$. 
Under Assumption \ref{ass:distribution} below, \eqref{eq:conditionalQuantile} is classically the unique minimizer $f^*_{\tau}$ of \eqref{eq:riskPinballLoss}, see Proposition 3.9 in \cite{steinwart2008support}.
\begin{assumption}
\label{ass:distribution}
  Let $p_{Y\mid Z}(y)$ be the density of the supposedly continuous conditional distribution of $Y$ given $Z$ w.r.t Lebesgue measure and assume that there exist constants $\nu>0$ and $\delta>0$ such that we have with probability $1$:
    \begin{equation}
      \forall y \in [Q_Y(\tau \mid Z) - \delta, Q_Y(\tau \mid Z) + \delta],\;   p_{Y\mid Z}(y) \geq \nu.
    \end{equation}
\end{assumption}
Assuming that $n\geq1$ independent copies $(Y_1,Z_1), $ $\ldots, (Y_n,Z_n)$ of the generic pair $(Y,Z)$ are observed, the Empirical Risk Minimization (ERM) principle suggests to minimize a statistical version of \eqref{eq:riskPinballLoss}, namely 
\begin{equation}\label{eq:riskPinballLoss_emp}
    \widehat{\mathcal{R}}_n (f) := \frac{1}{n} \sum_{i=1}^n \rho_\tau(Y_i - f(Z_i)),
\end{equation}
over a class of functions $\mathcal{F}$ of controlled complexity (\textit{e.g.} linear functions, neural networks) but hopefully rich enough to contain a reasonable approximant of $Q_Y(\tau \mid z)$. Depending on the class $\mathcal{F}$ chosen, various algorithms for computing (approximately) an empirical pinball loss minimizer $\hat{f}_\tau(z)\in \argmin_{f \in \mathcal{F}} \widehat{\mathcal{R}}_n (f)$ have been proposed in the literature. Among these, we can mention kernel-based approaches \citep{taktak12}, with a detailed theoretical analysis provided by \cite{steinwart2007}, 
tree-based methods (quantile regression forests) \citep{10.5555/1248547.1248582}, and more recently, quantile regression with ReLU neural networks \citep{JMLR:v23:21-0309}. Alongside computational aspects, upper confidence bounds for the excess of risk $\mathcal{E}(\hat{f}_{\tau})=\mathcal{R}(\hat{f}_\tau) - \min_{f\in \mathcal{F}}\mathcal{R}(f)$ of order $O_{\mathbb{P}}(1/\sqrt{n})$ can be established using standard tools from empirical process theory, subject to appropriate complexity assumptions for $\mathcal{F}$. In addition, when $f^*_{\tau}:=\argmin_{f\in \mathcal{F}}\mathcal{R}(f)$ coincides with $Q_Y(\tau\mid \cdot)$, it is shown by \cite{steinwart2007} that the calibration inequality for the pinball loss
\begin{equation}
\label{eq:steinwart}
\|f - f^*_\tau\|_{L_1} \leq c \sqrt{\mathcal{R}(f) - \mathcal{R}(f^*_\tau)},
\end{equation}
holds under conditions that are ultimately not very restrictive for some $c>0$, denoting by $\|g\|_{L_1}$ the $L_1$-norm of any integrable function $g:\mathcal{Z}\to \mathbb{R}$.
Under appropriate complexity assumptions for $\mathcal{F}$ and the variance control condition below
\begin{equation*}
\label{eq:assumption_steinwart}
\forall f\in \mathcal{F},\; \mathbb{E}[(\rho_\tau(Y-f(Z)) - \rho_\tau(Y-f_{\tau}^*(Z)))^2] \leq c'\cdot \mathcal{E}(f)^\theta,
\end{equation*}
for some constants $\theta\in [0,1]$ and $c'>0$, Theorem 2.6 in \cite{steinwart2007} shows that the empirical estimator $\hat{f}_{\tau}$ can achieve fast learning rates of order $O_{\mathbb{P}}(n^{-\gamma})$ for some $\gamma < 1$, depending on $\theta$ in particular. Combined with \eqref{eq:steinwart} this implies that the $L_1$-deviation $\|\hat{f}_{\tau} - f^*_{\tau}\|_{L_1}$ can approach order $1/\sqrt{n}$ with high probability.

The empirical pinball loss \eqref{eq:riskPinballLoss_emp} is scale-dependent, varying with both the target distribution and the quantile level $\tau$. To ensure interpretability across datasets, we adopt the $D^2_\tau$ goodness-of-fit measure \cite{Koenker01121999}:
\begin{equation}
\label{eq:ri_tau}
D^2_\tau = 1 - \frac{\widehat{\mathcal{R}}_n (f_\tau)}{\widehat{\mathcal{R}}_n (f_\tau^\mathrm{ref})},
\end{equation}
where $f_\tau^\mathrm{ref} \equiv \widehat{Q}_Y(\tau)$ denotes the unconditional $\tau$-quantile. This metric generalizes the $R^2$ statistic to the quantile framework; positive values indicate improvement over the marginal baseline, with $D^2_\tau$ approaching unity for a perfect fit. Additionally, model calibration is assessed via the empirical coverage $\hat{c} = n^{-1} \sum_{i=1}^n \mathbb{I}\{y_i \leq \hat{f}_{\tau}(z_i)\}$, which should ideally satisfy $\hat{c} \approx \tau$ for a well-specified model.

\subsection{Pairwise Learning and $U$-statistics}
\label{section:u-stats}
In many problems such as metric learning \citep{BHS15}, similarity scoring \citep{VBC18}, ranking \citep{clemencon_empirical} or clustering \citep{CL11}, for instance, natural empirical risk measures are expressed as pairwise averages. This is referred to as \textit{pairwise learning}. In this case, performance criteria are no longer i.i.d. averages like \eqref{eq:riskPinballLoss_emp} but take the form of their simplest extensions, $U$-statistics, see \citep{hoeffding1948}. Precisely, let $V_1, \dots, V_n$ be $n\geq 1$ independent copies of a generic r.v. $V$ taking its values in some measurable space $\mathcal{V}$. The statistic
\begin{equation}\label{eq:U_stat}
U_n(h) = \frac{1}{n(n-1)} \sum_{i \neq j} h(V_{i}, V_{j}),
\end{equation}
where $h : \mathcal{V}^2 \to \mathbb{R}$ is a measurable symmetric function such that $\mathbb{E}[h^2(V_1,V_2)]<+\infty$, is a $U$-statistic of order $2$ with kernel $h(v,v')$ based on the $V_i$'s. Among all unbiased estimators of the parameter $\theta(h) = \mathbb{E}[h(V_1, V_2)]$, it is the most efficient, \textit{i.e.} the one with minimal variance, see \citep{Ser80}. 

\begin{example}
\label{ex:1}
({\sc Metric learning})  
Let $V=(Y,Z)$, where $Y$ is a discrete label and $Z$ is a covariate in $\mathcal{Z}$. The goal is to learn a metric $\delta$ on $\mathcal{Z}$ such that points with the same label are close and points with different labels are far apart, as in \citep{BHS15}. This can be formalized as a pairwise learning problem with risk
$$
\mathcal{M}(\delta) = \mathbb{E}\big[\psi((\delta(Z,Z')-1)(2 \mathbb{I}_{\{Y=Y'\}}-1))\big],
$$
where $\psi(u)$ is a convex loss function upper bounding $\mathbb{I}_{\{u\geq 0\}}$, e.g., the hinge loss $\max(0,\; 1+u)$. The empirical risk over i.i.d. examples $(Y_1,Z_1),\dots,(Y_n,Z_n)$ is a $U$-statistic with symmetric kernel
$$
h_\delta\big((y,z),(y',z')\big) = \psi((\delta(z,z')-1)(2\mathbb{I}_{\{y=y'\}}-1)).
$$
\end{example}
The main difficulty encountered when analyzing the fluctuations of the deviation $U_n(h)-\theta(h)$, uniformly over a class $\mathcal{H}$ of kernels lies in the dependence structure of the averaged terms in \eqref{eq:U_stat}, which are no longer i.i.d. r.v.'s. As shown by \cite{delapena1999} and \cite{clemencon_empirical}, (first and second) Hoeffding decompositions combined with decoupling techniques allow sharp concentration results to be established for such collections of $U$-statistics, called $U$-processes. These tools have been used to study the performance of pairwise learning for various problems, such as supervised similarity scoring \citep{VBC18} or ranking \citep{JMLR:v17:14-265}, and derive generalization bounds for minimizers of empirical risks taking the form \eqref{eq:U_stat}. In some of these problems, the reduced variance property of $U$-statistics allows us to demonstrate that bounds faster than $1/\sqrt{n}$ hold true under mild conditions, see \citep{CLEMENCON201442} in the context of clustering. The purpose of this article is precisely to extend this type of analysis to pairwise quantile regression, as formulated in section \ref{sec:main}.

\paragraph{Scalability and Incomplete $U$-statistics}
While averaging across all pairs minimizes the variance under the zero bias constraint, the resulting computational cost, of order $O(n^2)$, may be prohibitive when $n$ is large. As originally proposed in \cite{Blom76} (see also \cite{Janson84, lee_incomplete_1982}), averaging over a set of $B\geq 1$ sampled pairs $(V_{i_1},V_{j_1}),\; \ldots,\; (V_{i_B},V_{j_B})$, drawn with replacement from the $n(n-1)/2$ observed pairs $\{(V_i,V_j):\; 1\leq i<j\leq n\}$, may lead to an unbiased Monte-Carlo estimate, referred to as an \textit{incomplete $U$-statistic},
\begin{equation}\label{eq:MC_est}
\bar{U}_B(h) = \frac{1}{B} \sum_{b=1}^B h\bigl(V_{i_b}, V_{j_b}\bigr),
\end{equation}
that offers a satisfactory compromise between reduced variance and computational cost for an appropriate choice of $B$.
In particular, \citet{JMLR:v17:15-012} proposed to replace the pairwise empirical risk with a Monte Carlo estimate involving only $\mathcal{O}(n)$ pairs, demonstrating that the learning rate $\mathcal{O}(1/\sqrt{n})$ is preserved for minimizers. Regarding the minimization procedure itself, \citet{papa_sgd_2015} introduced a SGD variant based on incomplete $U$-statistics, showing that sampling pairs at each iteration yields scalable learning procedures with theoretical guarantees comparable to those of full $U$-statistics.

\paragraph{Similarity Scoring in Facial Recognition} 
\label{intro_fr}
The development of FR systems based on machine-learning has been the subject of extensive research over the past decade, although some important challenges remain (\textit{e.g.} fairness guarantees), as highlighted by \citep{Grother2022FRVT8}. While \cite{dutta_quality} attempts to predict model performance based on image quality, and \cite{bolme_data_2024-1} analyzes the influence of covariates such as resolution and subject distance in aerial and long-range biometric data, understanding in depth why and when FR systems fail remains a largely open problem. These systems rely on measuring the similarity between face embeddings to determine identity matches.
In practice, a deep CNN classification model is trained to predict the identity $Y$ of an individual whose face is shown in an image $X$ based on a training database containing  labeled images of same format, \textit{i.e.} in $\mathbb{R}^{h\times w\times c}$, where by $h \times w$ is meant the size of the images and by $c$ the number of color channel, the labels corresponding to the $K\geq 2$ identities present in the database indexed by $y\in\{1,\; \ldots,\; K\}$. The embedding produced by the penultimate layer $g:x\in\mathbb{R}^{h\times w\times c}\mapsto g(x)\in \mathbb{R}^p$ is then used to represent any image $X$ in $\mathbb{R}^{h\times w\times c}$, the latent ($p$-dimensional) representation being thus $g(X)$.
Most famous examples of implementations of this approach include ArcFace \citep{deng_arcface_2019} and CosFace \citep{wang_cosface_2018}. The encoding function $g$ is expected to embed images corresponding to the same identities close to one another.
Typically, the similarity between two face images $X$ and $X'$ is measured using the cosine similarity between their embeddings:
\begin{equation}
\label{eq:cosine-similarity}
s(X,X') = \frac{g(X)^{^\intercal} g(X')}{\|g(X)\| \cdot \|g(X')\|},
\end{equation}
where $\|.\|$ means the Euclidean norm on $\mathbb{R}^p$. The performance of $\eqref{eq:cosine-similarity}$ is assessed using the two competing metrics
\begin{eqnarray*}
    \text{FRR}(t) &=&  \mathbb{P}\{s(X,X')\leq t \mid Y= Y'\}, \\
    \text{FAR}(t) &=& \mathbb{P}\{s(X,X')> t \mid Y\neq  Y'\},
\end{eqnarray*}
referred to as the False Rejection Rate (FRR) and the False Acceptance Rate (FAR) respectively, at threshold $t\in (-1,+1)$. In practice, $t$ is selected to achieve a balance between security (low FAR) and usability (low FRR). Typically, estimators of $\text{FAR}(t)$ (respectively, of $\text{FRR}(t)$) are incomplete $U$-statistics \eqref{eq:MC_est} obtained by drawing pairs from the population of 'impostor pairs' (respectively, from the population of 'genuine pairs'). Impostor pairs are pairs with different identities, while genuine pairs are pairs with the same identity. Since a high-dimensional feature vector $Z$ (\textit{e.g.} image quality measurements, hair color, age group) is associated to any image $X$ (and to the individual depicted on it), the similarity score $s(X,X')$ between two face images $X$ and $X'$ possibly depend on the pair $(Z,Z')$ formed by their covariates. In order to improve the performance of FR systems, it is crucial to understand the impact of the covariate pair on high impostor scores on the one hand, and on low genuine scores on the other. This is a natural application area for pairwise quantile regression, which will be studied in detail in section \ref{sec:exp}, after establishing theoretical guarantees in this specific context in the following section.


\section{Pairwise Quantile Regression}\label{sec:main}

In this section, the objective of the pairwise learning problem under study is rigorously formulated, and the statistical framework considered to solve it, namely by minimizing a specific $U$-statistic, is also described. A probabilistic analysis of the predictive and estimation performance is then conducted, demonstrating that learning rates faster than those of the central limit theorem can be attained, under assumptions that are not very restrictive in practice, due to the form of the statistical counterpart of the specific risk to be minimized.

\subsection{Statistical Learning Framework}

Let $V=(X,Z)$ be a random tuple defined on a probability space $(\Omega,\mathcal{F},\mathbb{P})$, where the r.v. $X$ takes its values in a space $\mathcal{X}$ and $Z$ denotes an associated vector of covariates. Consider an independent copy of it, $V'=(X',Z')$, as well as a real-valued scoring function $s:(x,x')\in \mathcal{X}^2\mapsto s(x,x')$ hopefully quantifying the similarity between two instances $x$ and $x'$ in the feature space $\mathcal{X}$ (\textit{e.g.} cosine similarity, Mahalanobis distance, Mercer kernel). The goal pursued here is to regress the quantiles of the similarity score $s(X,X')$ on the pair of covariates $(Z,Z')$, namely to recover the conditional quantiles of $s(X,X')$ given $(Z, Z')$ at specific levels $\tau\in (0,1)$:
\begin{equation}\label{eq:pair_quantile_funct}
Q_{s}(\tau \mid Z, Z'):= \inf_{y \in \mathbb{R}} \{F_{s(X, X') \mid Z, Z'}(y) \geq \tau\},
\end{equation}
where $F_{s(X, X') \mid Z, Z'}$ means the conditional cdf of $s(X,X')$ given the pair $(Z,Z')$.  Let us assume that the target level $\tau$ is fixed. As recalled in subsection \ref{section:quantile}, accurate approximations of \eqref{eq:pair_quantile_funct} can be obtained, under appropriate conditions, by minimizing the risk
\begin{equation}\label{eq:pair_risk}
R(q) \;=\; \E[H_q(V, V')],
\end{equation}
where $H_q(V,V') := \rho_{\tau}( s(X,X') - q(Z,Z'))$, over a class $\mathcal{Q}$ of symmetric measurable functions $q:\mathcal{Z}^2\to \mathbb{R}$ (\textit{i.e.} $q(z,z')=q(z',z)$ for all $(z,z')\in \mathcal{Z}^2$), sufficiently rich to contain (an approximation of) \eqref{eq:pair_quantile_funct}, while remaining of controlled complexity. 
Statistical learning here is supposedly based on the observation of a set of $n\geq 2$ training examples $\{V_i=(X_i,Z_i):\; i=1,\; \ldots,\; n\}$, assumed to be independent copies of the generic tuple $V=(X,Z)$. In this context, a natural statistical counterpart of \eqref{eq:pair_risk} (which is of minimal variance among all unbiased estimators based on the $V_i$'s) is the $U$-statistic of degree $2$ with (symmetric) kernel $H_q$:
\begin{equation}\label{eq:emp_pair_risk}
\widehat{R}_n(q):= \frac{2}{n(n-1)} \sum_{i < j} H_q(V_i,V_j).
\end{equation}
The learning strategy under study here consists in minimizing the empirical pairwise pinball loss \eqref{eq:emp_pair_risk} over the class $\mathcal{Q}$ and we note $\hat{q}_\tau = \arg \min_{q \in \mathcal{Q}} \widehat{R}_n(q)$ its empirical minimizer. We point out that the development of numerical methods for solving (approximately) this minimization problem is beyond the scope of this paper. Indeed, we focus on the original statistical aspect (the reduced variance property of the risk resulting from pairwise averaging) leading to fast learning rates as we shall see below. Note incidentally that the numerical methods used in pointwise quantile regression (\textit{e.g.} linear, SVM, random forest) can be readily adapted to the pairwise case. Before investigating the performance of such minimizers, a few comments are in order.\\
The guarantees below impose no parametric form on $\mathcal{Q}$; only its complexity enters the bounds, measured here by the finite VC dimension (Rademacher averages would serve equally well). Computing $\widehat{q}_\tau$ reuses the exact procedures of pointwise quantile regression; the only pairwise-specific requirement is to enforce the symmetry $q(Z, Z') = q(Z', Z)$, obtained by composing the functions used in the pointwise case by the transform $\phi (Z, Z') = (Z + Z', |Z - Z'|)$.
\begin{remark}{\sc (On level $\tau$)} It should be noted here that the level $\tau\in (0,1)$ of the conditional quantiles we are seeking to learn is assumed to be fixed. Our framework therefore does not cover \textit{extreme quantile regression} \citep{chernozhukov2017extremal}, where $\tau$ could depend on the number $n\geq 2$ of training examples and approach $0$ or $1$ as $n$ increases. We also do not consider here the simultaneous/multitask learning of quantile regression functions at several given $\tau$ levels, faced with the problem of quantile crossover \citep{sangnier_joint_2016}. Extending the theoretical results relating to pairwise quantile regression established below to these specific frameworks could be the subject of future work.
\end{remark}
\begin{example}{\sc (Additive noise model)}
\label{ex:pairwise_additive_noise}
Consider the pairwise noise model $\displaystyle{s(X,X') = m(Z,Z') + \sigma(Z,Z')\epsilon}$, where $m: \mathcal{Z}^2 \to \mathbb{R}$ and $\sigma: \mathcal{Z}^2 \to \mathbb{R}_+^*$ are Borel measurable functions and $\epsilon$ is a centered r.v. independent of $(Z,Z')$ with strictly increasing and continuous cdf $F_\epsilon$. Then, the conditional distribution of $s(X,X')$ given $(Z, Z')$ has density 
$p_{s \mid Z, Z'}(t) = F'_\epsilon( (t - m(Z,Z'))/\sigma(Z,Z'))/\sigma(Z,Z')$ and the target predictive function \eqref{eq:pair_quantile_funct} has the explicit form
$Q_s(\tau \mid Z, Z') = m(Z, Z') + \sigma(Z, Z') F_\epsilon^{-1}(\tau)$.
\end{example}
\subsection{Rate Bound Analysis - Fast Learning}

The performance of minimizers $\hat{q}_{\tau}(z,z')$ of the empirical pairwise pinball loss \eqref{eq:emp_pair_risk} over the class $\mathcal{Q}$ can be first assessed by establishing upper confidence bounds for the excess of risk $\mathcal{E}(\hat{q}_{\tau}) := R(\hat{q}_{\tau}) - \inf_{q\in \mathcal{Q}}R(q)$. In this purpose, assumptions about the class $\mathcal{Q}$ are necessary.
\begin{assumption}[Bounded Vapnik-Chervonenkis (VC) class]\label{hyp:Q}
The class $\mathcal Q$ is a bounded VC class of symmetric functions with VC-dimension $L < + \infty$: there exists $B>0$ such that $|q(z,z')| \le B$
for all $z$, $z'$ in $\mathcal Z$ and any $q$ in $\mathcal{Q}$.
\end{assumption}
If we choose the cosine similarity \eqref{eq:cosine-similarity} as similarity function, it is natural to stipulate the uniform boundedness condition above with $B=1$. For simplicity, we assume here that $\vert\vert s\vert\vert_{\infty}:=\sup_{(x,x')\in \mathcal{X}^2}\vert s(x,x')\vert<\infty$. In this case, the kernels $H_q$, $q\in \mathcal{Q}$, are uniformly bounded as well: $\displaystyle{ \forall q \in \mathcal Q,\; |H_q(V,V')| \le B + \vert\vert s\vert\vert_{\infty}}$. As indicated in Appendix~\ref{app:relax_A1}, the boundedness assumption can be relaxed at the cost of certain technical subtleties. 

While the classic bound
$\mathcal{E}(\hat{q}_{\tau})\leq 2\sup_{q\in \mathcal{Q}}\vert \widehat{R}_n(q)-R(q) \vert$
allows us to establish directly nonasymptotic bounds of order $O_{\mathbb{P}}(1/\sqrt{n})$ for the excess of risk $\mathcal{E}(\hat{q}_{\tau})$ using concentration inequalities for $U$-processes, see \citep{major2006}, it does not exploit the reduced variance property of empirical pairwise losses of the form \eqref{eq:U_stat}. In binary classification, the flagship problem in statistical learning theory, it is well known that learning rates faster than the ‘universal’ rate $1/\sqrt{n}$ can be established when the variance of the empirical excess of risk can be controlled by a power of its expectation (which is guaranteed under specific assumptions about the posterior probability), see \citep{10.1214/009053606000000786, 10.1214/aos/1079120131}. While similar results have been proved in the context of pointwise quantile regression under restrictive noise conditions as recalled in subsection \ref{subsec:point_QR}, see \citep{steinwart2007, Steinwart_2011}, it is quite remarkable in the case of pairwise quantile regression that the ‘parametric’ learning rate $1/n$, up to a logarithmic factor, is achieved under a very mild assumption<, namely the same one that guarantees the uniqueness of the minimization problem $\min_q R(q)$, \textit{i.e.} the pairwise analogue of Assumption \ref{ass:distribution} below.
\begin{assumption}
\label{ass:density} Suppose that $s(X,X')$'s conditional distribution given $(Z,Z')$ is continuous with density $p_{s \mid Z,Z'}(t)$. 
There exist constants $\nu>0$, $\delta>0$ s.t. for all $(z,z') \in \mathcal{Z}^2$: 
$$
\forall t \in [Q_{s}(\tau \mid z,z') - \delta, \, Q_{s}(\tau \mid z,z') + \delta],\; p_{s \mid Z,Z'}(t) \;\ge\; \nu.
$$
\end{assumption}
In the setting of Example~\ref{ex:pairwise_additive_noise}, if $\epsilon$ admits a continuous, strictly positive density $f_\epsilon$, then Assumption~\ref{ass:density} is satisfied.
For facial recognition, the assumption is reasonable for two reasons. First, the covariates $Z$ do not fully determine the raw image or the high-dimensional embeddings from which $s(X,X')$ is computed; the residual biometric information acts as noise, making the conditional score distribution continuous rather than degenerate. Second, the embedding network (e.g., ArcFace) and the cosine similarity are continuous maps, so the conditional density $p_{s \mid Z,Z'}$ is naturally smooth.
In the footsteps of the fast rate analysis carried out by \cite{clemencon_empirical}, it should be noted that $\hat{q}_{\tau}(z,z')$ is also a minimizer of the empirical excess risk
\begin{equation}\label{eq:emp_excess}
\widehat{\mathcal{E}}_n(q):=\widehat{R}_n(q)-\widehat{R}_n(Q_{s}(\tau \mid \cdot)),
\end{equation}
which is also a $U$-statistic of degree $2$ based on the $V_i$'s, whose (bounded symmetric) kernel is $K_q(v,v')=H_q(v,v')-H_{Q_{s}(\tau \mid \cdot)}(v,v')$. The key observation made by \cite{clemencon_empirical} is that the $U$-statistic \eqref{eq:emp_excess} can be uniformly approximated by the i.i.d. average $2T_n(q)=(2/n)\sum_{i=1}^nk_q(V_i) - 2\mathcal{E}(q)$, where $k_q(v)=\mathbb{E}[K_q(v,V)]$ for all $v\in \mathcal{X}\times \mathcal{Z}$, which is the leading term of its Hoeffding decomposition \citep{hoeffding1948}. This can be established by means of concentration inequalities for the degenerate $U$-process $\{\widehat{\mathcal{E}}_n(q)-2T_n(q):\; q\in \mathcal{Q}\}$. Minimizing \eqref{eq:emp_excess} being thus approximately equivalent to minimizing $T_n(q)$, the variance term that comes into play is $\Var(k_q(V))$, smaller than $\Var(K_q(V,V'))$ of course, and which is always bounded by the risk excess, as the result stated below shows.
\begin{proposition}
\label{prop:alpha_condition}
Suppose that Assumption~\ref{ass:density} is fulfilled. For all $q \in \mathcal Q$, we have:
$$
\Var(k_q(V)) \;\le\; C_{\mathrm{var}} \, \mathcal{E}(q),
$$
where $C_{\mathrm{var}} = 2M_\tau^2/\nu$ and $M_\tau = \max(\tau, 1-\tau)$.
\end{proposition}
Refer to \ref{proof:alpha_condition} for the technical proof. We point out that this control of the variance (of the approximation of) the empirical excess risk by its expectation echoes the condition stipulated by \cite{steinwart2007} and recalled in subsection \ref{subsec:point_QR}. It is noteworthy however that, unlike the case of pointwise quantile regression, the strongest control (corresponding to $\theta=1$) always holds true in the pairwise situation. The following result thus mainly follows from Proposition \ref{prop:alpha_condition} and Corollary 6 in \cite{clemencon_empirical}, and establishes bounds of order $\log(n)/n$ for the risk excess of minimizers $\hat{q}_{\tau}$ of \eqref{eq:emp_pair_risk}.
\begin{theorem}
\label{thm:pinball-uerm}
Suppose that Assumptions \ref{hyp:Q} and \ref{ass:density} are fulfilled. Assume also that $Q_s(\tau\mid \cdot)\in \mathcal{Q}$.
Then there exists a universal constant $C>0$ such that for all $\delta\in(0,1)$, we have with probability at least $1-\delta$,
\begin{equation}\label{eq:bound1}
\mathcal E(\hat q_\tau) \leq     C \left(
\frac{L\log(n)}{n}
+
\frac{(B+\vert\vert s\vert\vert_{\infty})\log(1/\delta) + L}{n}\right).
\end{equation}
\end{theorem}
The technical proof is given in \ref{app:proof-thm1}.  In addition, as under Assumption \ref{ass:density}, we have for any measurable function $q$:
\begin{equation}\label{eq:q_bound}
\|q - Q_s(\tau\mid \cdot)\|_{L_2(P)}^2 \leq \frac{2}{\nu}\mathcal{E}(q),
\end{equation}
one immediately deduces from Theorem \ref{thm:pinball-uerm} that empirical minimizers $\hat{q}_{\tau}$ estimate the conditional quantile $Q_s(\tau\mid \cdot)$ at rate $\sqrt{\log(n)/n}$ w.r.t the $L_2$-norm, as formulated below. 

\begin{corollary}
\label{corr:l2_distance}
Suppose that Assumptions \ref{hyp:Q} and \ref{ass:density} are fulfilled. Assume also that $Q_s(\tau\mid \cdot)\in \mathcal{Q}$.
Then there exists a universal constant $C>0$ such that for all $\delta\in(0,1)$, we have with probability at least $1-\delta$,
\begin{multline}\label{eq:bound2}
    \|\hat{q}_\tau - Q_s(\tau \mid \cdot)\|_{L_2(P)} \leq \\ \sqrt{\frac{2C}{\nu}\Bigg(\frac{L\log(n)}{n} + \frac{(B+\vert\vert s\vert\vert_{\infty}) \log(1/\delta) + L}{n}\Bigg)}.
\end{multline}

\end{corollary}
The proof can be found in Appendix \ref{app:proof-corollary-l2}. \\

This acceleration is not merely a consequence of averaging over $O(n^2)$ pairs rather than $n$ points: these pairs are dependent (each $V_i$ appears in $n-1$ of them), so for fixed $q$ the CLT for non-degenerate U-statistics gives fluctuations of order $O_P(1/\sqrt{n})$, which generic empirical-process arguments cannot improve. The gain to $1/n$ comes instead from the variance/margin mechanism behind fast rates in classification \citep{massart_2007}: it suffices that the variance of the empirical excess of risk be controlled by a power of its expectation. What is specific to the pairwise setting is that this control holds automatically, in its strongest form ($\theta = 1$), since the variance of the projections $k_q(V_i)$ is bounded by the excess risk itself (see Proposition \ref{prop:alpha_condition}). The improvement is thus genuine, not an artifact of the larger number of pairs. \\

As shown in the Appendix \ref{app:approximation_error}, the condition that the conditional quantile at level $\tau$, $Q_s(\tau\mid \cdot)$, belongs to the class $\mathcal{Q}$ can be easily relaxed, thereby adding the model bias term $\inf_{q\in \mathcal{Q}}R(q)-R(Q_s(\tau\mid \cdot))$ to the upper bounds in \eqref{eq:bound1} and in \eqref{eq:bound2}. Ideally, the class $\mathcal{Q}$ should be chosen so as to balance the model bias term with the stochastic error bound in \eqref{eq:bound1}. Model selection techniques by additive complexity penalization could be implemented to select the class $\mathcal{Q}$ from a collection of classes
achieving this objective following the approach described in Appendix \ref{app:model_selection}. The same bounds as those used in the analysis above could be classically used to establish oracle inequalities for such complexity regularization (or structural risk minimization) methods.


\section{Numerical Experiments}\label{sec:exp}
In this section, we empirically evaluate the performance of our pairwise quantile regression framework. \footnote{The code is available at \url{https://github.com/Romain-Therezien/Pairwise-QR}} We begin with controlled synthetic examples that allow us to verify theoretical properties and illustrate the behavior of the estimator under known conditions. Subsequently, we apply our method to a real-world facial recognition task, demonstrating its ability to capture extreme pairwise similarity scores and provide interpretable insights.
\subsection{Synthetic Examples}
\label{sec:synthetic}

To validate our theoretical findings, we consider a synthetic heteroskedastic setup. Let 
$Z \sim \mathcal{U}(-1,1)$ with conditional parameters $\mu_Z = 0.1 Z$ and $\sigma_Z = 0.3 |Z|$, 
and generate observed features as $X = Z + \epsilon_Z, \quad \epsilon_Z \sim \mathcal{N}(\mu_Z, \sigma_Z).$
We target the pairwise score function $s(X, X') = \sin(X + X')$. Conditional on $(Z,Z')$, 
$s$ has a smooth and strictly positive density near its conditional quantile, satisfying 
Proposition~\ref{prop:alpha_condition}. Neural networks (NN) are used as the function class 
$\mathcal{Q}$, which satisfies the relaxed Assumption~\ref{hyp:Q} (see Appendix~\ref{app:relax_A1}), 
ensuring that Theorem~\ref{thm:pinball-uerm} applies. Pairs $(Z, Z')$ are concatenated then use as input for the models. 

\paragraph{Results.}
Fig.~\ref{fig:Setup-1-Pinball_loss_across_quantiles} demonstrates that the U-statistic ERM 
approach generalizes stably across quantile levels $\tau \in (0,1)$. The results demonstrate stable generalization, with test performance closely matching training performance across the full range of quantiles.
\begin{figure}[h!]
    \centering
    \includegraphics[width=0.5\textwidth]{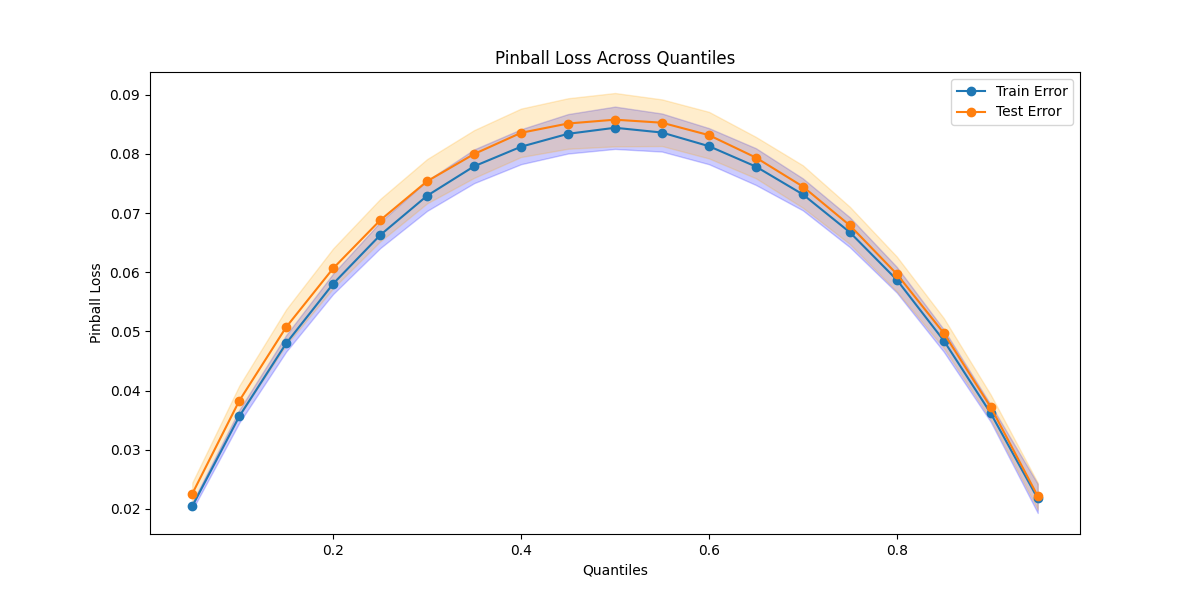}
    \caption{Pinball Loss for different Quantiles.}
    \label{fig:Setup-1-Pinball_loss_across_quantiles}
\end{figure}
In Appendix Fig.~\ref{fig:Model_Comparison}, we compare our NN approach with LightGBM and Gradient Boosting using Mean Absolute Error (MAE)  relative to the true conditional quantile (estimated via Monte Carlo). The NN consistently outperforms tree-based baselines, capturing the smooth pairwise dependencies more accurately.

To qualitatively assess the model's ability to capture heteroskedasticity and complex noise structures, we visualize the predicted conditional quantile surfaces in Figure~\ref{fig:3D-plot_Setup_1}. The surfaces illustrate the model's ability to capture the underlying heteroskedastic transformation of the latent Gaussian variables under the pairwise scoring function $s(X, X')$.
\begin{figure}
    \centering
    \includegraphics[width = 0.5\textwidth, trim = 85 25 45 25, clip]{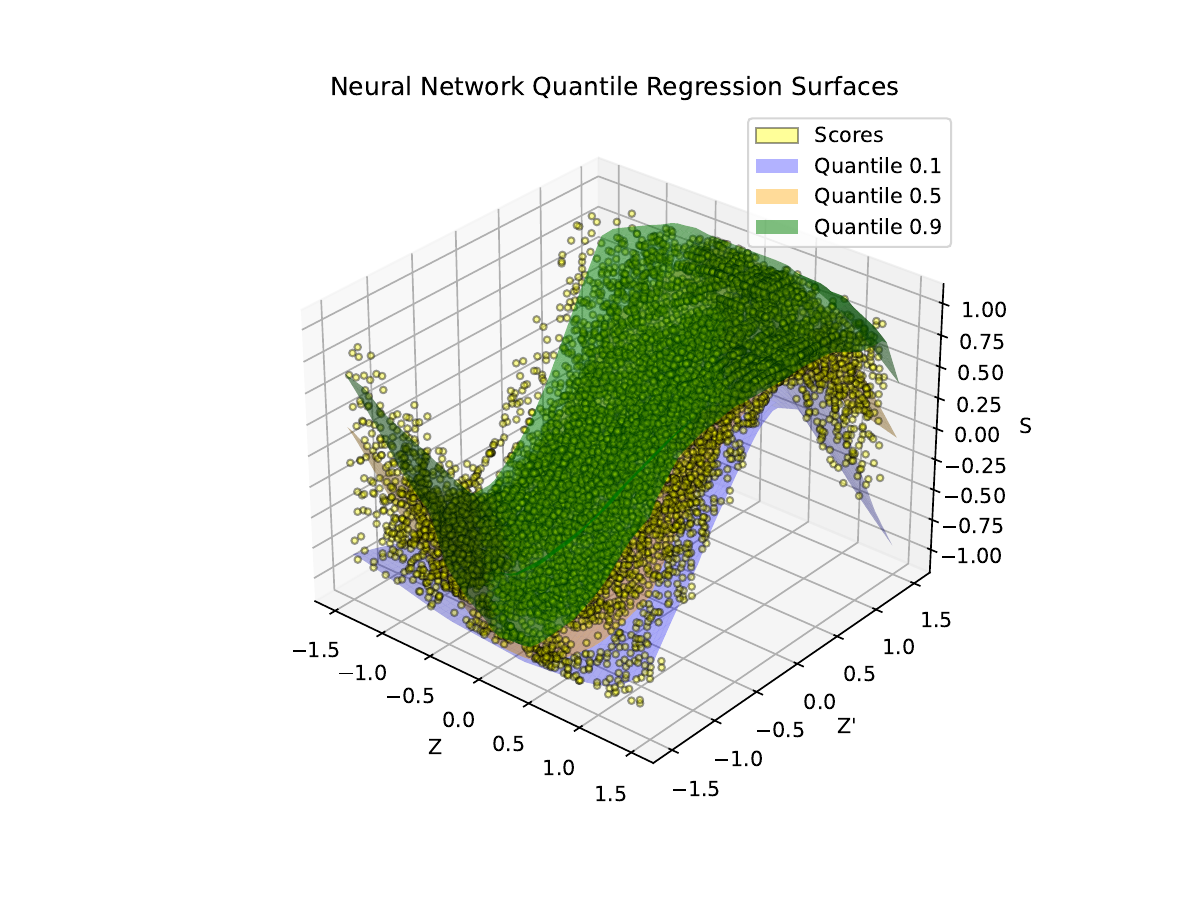}
    \caption{Conditional quantile ($\tau = 0.1, 0.5, 0.9$).}
    \label{fig:3D-plot_Setup_1}
\end{figure}

Additional experiments, including further implementation details and evaluations on incomplete U-statistics, are reported in Appendix~\ref{app:synthetic}.

\subsection{Application to Facial Recognition}\label{subsec:FR}

We evaluate our pairwise quantile regression framework on a facial recognition task, 
where the goal is to predict similarity scores between pairs of face embeddings. 
Standard regression methods estimate the expected similarity, but extreme scores, where impostors appear unusually similar or genuine pairs appear dissimilar, are particularly consequential for recognition errors and occur relatively rarely. By modeling conditional quantiles of pairwise similarity scores, our approach targets these critical cases directly, providing accurate estimates of rare events and interpretable insights for practical decision-making. This setup naturally aligns with the U-statistics framework, as each similarity score is a pairwise statistic, allowing us to leverage both theoretical properties and practical interpretability. We will be using Shapley Values \citep{lundberg_unified_2017} to decompose the estimated conditional quantile $\hat{q}_{\tau}(Z, Z')$ into covariate-specific contributions, thereby highlighting the variables that most strongly determine it. This approach offers a notion of interpretability that is independent of the model and applicable to various function classes $\mathcal{F}$ and quantile levels $\tau\in (0,1)$. To the best of our knowledge, pairwise quantile regression has not been explored empirically, and therefore no direct comparisons with existing methods are available.

\paragraph{Dataset.} 
The dataset, included in the supplementary material and to be released publicly upon acceptance, 
contains $125{,}052$ genuine pairs and $1{,}149{,}498$ impostor pairs. Each image is annotated 
with covariates such as image quality, hair color, and other attributes (Table~\ref{tab:features}). 
To analyze false negatives (low similarity scores for genuine pairs) and false positives (high scores for impostor pairs), we treat them separately. 

Rather than using raw covariate pairs $(Z, Z')$ directly, we construct symmetric, interpretable features $(Z + Z', \, |Z - Z'|)$,
which ensures invariance to pair ordering $(Z \leftrightarrow Z')$ and facilitates interpretation. 
The sum $Z + Z'$ captures the aggregate effect of covariates across the pair, while the absolute 
difference $|Z - Z'|$ captures discrepancies between images, directly relating to similarity. 
These transformed features serve as inputs to our pairwise quantile regression model, enabling modeling of the conditional distribution of similarity scores.

\paragraph{Results.} 
We model the conditional distribution of similarity scores using a feedforward neural network with 
two hidden layers of $64$ and $32$ units (ReLU activations) and an output layer predicting the pairwise score. 
Separate models are trained for each quantile level $\tau$ using the pinball loss for $500$ epochs, 
with learning rate $0.001$ and batch size $64$. Hyperparameters were selected via grid search. 
Training separate models per quantile accounts for varying feature effects across quantiles. 
Neural networks were chosen for their superior performance compared to LightGBM and Gradient Boosting, consistent with observations from the synthetic experiments. Experiments using LightGBM and Gradient Boosting are presented in Appendix \ref{app:feature_validation}. 

Table~\ref{tab:fr_results} reports the relative improvement in quantile loss, $D^2$ 
(Equation~\eqref{eq:ri_tau}), compared to the constant quantile defined as the $\tau$-th quantile of 
$\{s(X_{i_1}, X_{i_2})\}, (i_1, i_2) \in \mathcal{D}$, the scores of dataset $\mathcal{D}$. 
N/A entries indicate quantiles irrelevant for the given pair type, as we focus on low scores for genuine pairs 
and high scores for impostor pairs.

\begin{table}[h]
\caption{$D^2$-score for facial recognition pairs. N/A indicates quantiles not relevant for the given pair type.}
\label{tab:fr_results}
\centering
\begin{tabular}{lcc}
\toprule
Quantile Level & Genuine Pairs & Impostor Pairs \\
\midrule
0.01 & 33.4\% & N/A \\
0.05 & 37.4\% & N/A \\
0.95 & N/A & 17.0\% \\
0.99 & N/A & 22.3\% \\
\bottomrule
\end{tabular}
\end{table}

As defined in \ref{eq:ri_tau}, $D^2$ represents the proportional reduction in “pinball” loss relative to an unconditional reference model. A $D^2$ score of $33.4\%$ or $37.4\%$ (for genuine pairs) indicates that our framework explains more than one-third of the variability in the distribution of similarity scores by leveraging the provided covariates (age, lighting, etc.). In the context of biometric scoring, where “global” variance is extremely high, these values represent a significant and effective capture of the conditional distribution. Furthermore, this qualitative assertion is objectively verified by Fig. \ref{fig:coverage}, which shows that our predicted quantiles align perfectly with the diagonal identity line. While $D^2$ quantifies the predictive power of the features, Fig. \ref{fig:coverage} demonstrates that the model is statistically well-calibrated, confirming that we have indeed captured the underlying distribution characteristics of the biometric engine. These results demonstrate that our framework effectively captures distributional characteristics of similarity scores, providing interpretable insights into extreme cases that are critical for recognition reliability.

\paragraph{Interpretability} 
To capture the complexity of similarity scores, we employ expressive neural networks and assess feature contributions post-hoc using Shapley values. This allows us to identify which input features most influence extreme similarity scores, low scores for genuine pairs and high scores for impostor pairs. Figures~\ref{fig:shap_values_pos_0.05} and~\ref{fig:shap_values_neg_0.95} show the top six features at extreme quantiles $\tau = 0.05$ and $\tau = 0.95$, respectively.

\begin{figure}[h]
    \centering
    \includegraphics[width=0.5\textwidth]{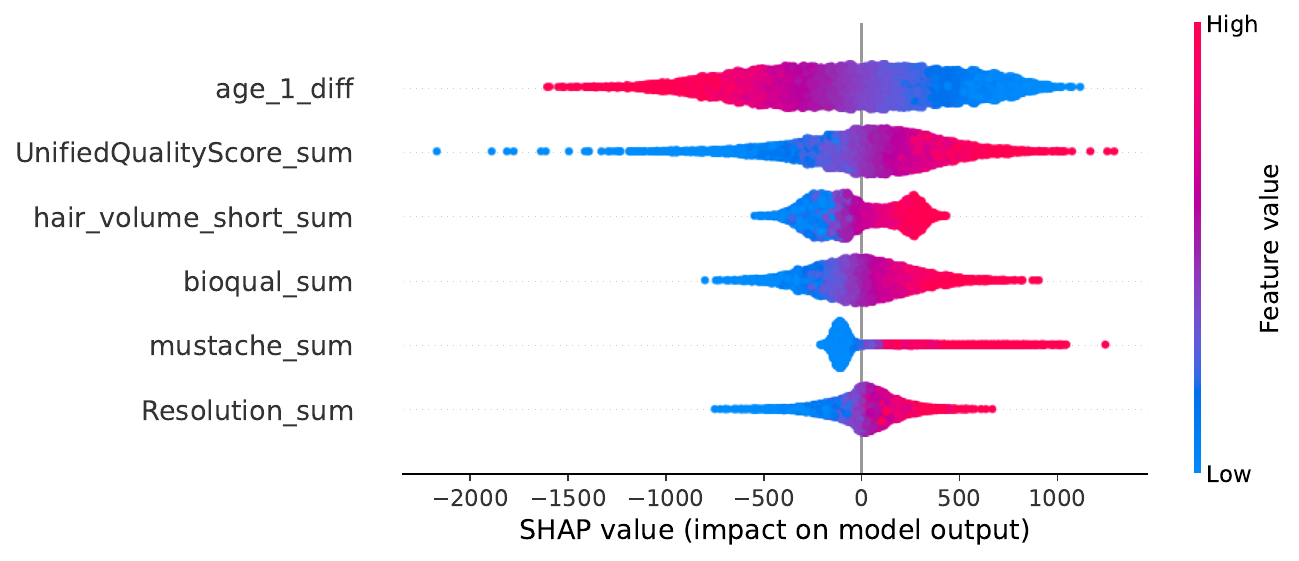}
    \caption{Shapley values for the six most influential features at the 5th percentile ($\tau = 0.05$) for genuine pairs. Larger age differences reduce similarity scores, while higher image quality increases them.}
    \label{fig:shap_values_pos_0.05}
\end{figure}

\begin{figure}[h]
    \centering
    \includegraphics[width=0.5\textwidth]{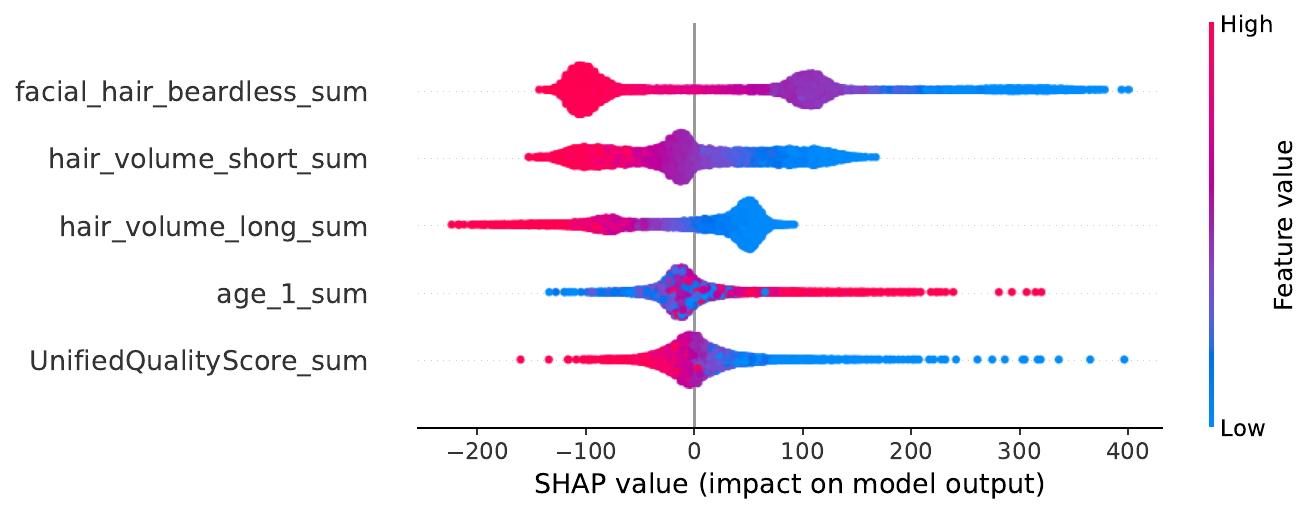}
    \caption{Shapley values for the six most influential features at the 95th percentile ($\tau = 0.95$) for impostor pairs. Lower-quality images increase predicted similarity, while greater differences in hair length reduce it.}
    \label{fig:shap_values_neg_0.95}
\end{figure} 

Analysis of Shapley values reveals that, for genuine pairs,  larger age differences reduce predicted similarity, while smaller differences increase it. The sum of image quality scores is also highly relevant, indicating that higher-quality images increase predicted similarity. For impostor pairs, images with similar hair length (either both short or both long) tend to have higher predicted similarity, while higher-quality images decrease similarity as they are easier to distinguish. 

Figure~\ref{fig:importance_heatmap_pos} further shows feature importance across quantiles for genuine pairs, highlighting that different features dominate at different levels. For example, the sum of image quality has a strong influence on lower quantiles but becomes less important at higher quantiles, whereas the difference of beard is primarily relevant for the highest quantiles. A more detailed discussion of feature impacts can be found in \ref{app:feature_impact}. By estimating conditional quantiles rather than only the mean, our approach characterizes the distributional behavior of extreme pairwise scores, which correspond to the most challenging cases for recognition.

\begin{figure}[h]
    \centering
    \includegraphics[scale=0.4, trim = 0 0 0 20, clip]{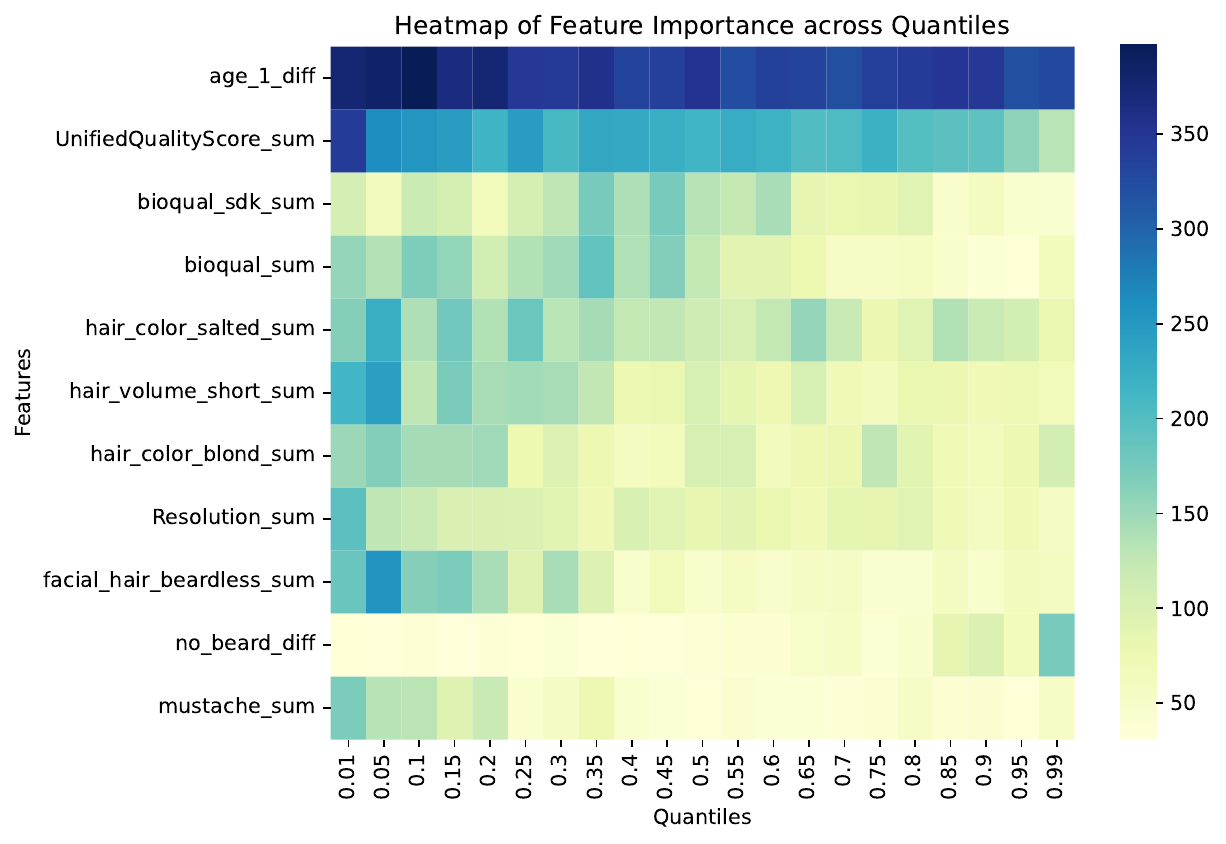}
    \caption{Heatmap of feature importance across quantile levels ($\tau$) for same identity pairs. Each row corresponds to a feature, and each column to a quantile. Darker colors indicate higher Shapley values.}
    \label{fig:importance_heatmap_pos}
\end{figure}
These insights can guide improvements in facial recognition models. Specifically, for each image pair described by its covariates, the model estimates a conditional quantile that can be used to recalibrate observed similarity scores. For example, if an impostor pair’s predicted conditional quantile deviates from its nominal level, the score can be adjusted via quantile alignment to better match the expected distribution.

Appendix~\ref{app:feature_validation} examines the stability of feature importance across additional quantiles and model classes, showing consistent identification of important factors across architectures.

\section{Conclusion}
In this paper, we presented a framework for pairwise quantile regression based on U-statistics and established fast convergence rates of $O(n^{-1})$ under a uniform lower bound on the conditional density near the target quantile. Applying this to facial recognition, we modeled similarity scores as functions of covariates. Our approach identifies important factors, such as the age difference, the image quality or the hair length, that are not captured by mean regressor models. In future work, this framework can guide adaptive calibration of similarity scores, enabling models to reduce demographic disparities while maintaining overall accuracy.

\subsection*{ACKNOWLEDGMENTS}
This research was fully funded by the French National Research Agency (ANR) in the framework of the FAR-SEE project (ANR-24-CE23-0921).

\bibliography{biblio_survie}      

@article{hoeffding1948,
author = {W. Hoeffding},
title = {{A Class of Statistics with Asymptotically Normal Distribution}},
volume = {19},
journal = {The Annals of Mathematical Statistics},
number = {3},
publisher = {Institute of Mathematical Statistics},
pages = {293 -- 325},
year = {1948},
}

@article{chernozhukov2017extremal,
  title={Extremal quantile regression},
  author={Chernozhukov, Victor and Fern{\'a}ndez-Val, Iv{\'a}n and Kaji, Tetsuya},
  journal={Handbook of Quantile Regression},
  pages={333--362},
  year={2017},
  publisher={Chapman and Hall/CRC}
}

@inproceedings{VBC18,
  title={A probabilistic theory of supervised similarity learning for pointwise ROC curve optimization},
  author={Vogel, R. and Bellet, A. and Cl{\'e}men{\c{c}}on, S.},
  booktitle={International Conference on Machine Learning},
  pages={5065--5074},
  year={2018},
  organization={PMLR}
}

@techreport{Grother2019,
    author = {Grother, P. and Ngan, M.},
    title = {{Face Recognition Vendor Test (FRVT) --- Performance of Automated Gender Classification Algorithms}},
    year = {2019}
}

@article{CLEMENCON201442,
title = {A statistical view of clustering performance through the theory of U-processes},
journal = {Journal of Multivariate Analysis},
volume = {124},
pages = {42-56},
year = {2014},
author = {S. Clémençon}
}

@article{clemencon_empirical,
author = {S. Cl{\'e}men{\c{c}}on and G. Lugosi and N. Vayatis},
title = {{Ranking and Empirical Minimization of U-statistics}},
volume = {36},
journal = {The Annals of Statistics},
number = {2},
publisher = {Institute of Mathematical Statistics},
pages = {844 -- 874},
year = {2008},
}

@article{JMLR:v23:21-0309,
  author  = {O. H. M. Padilla and W. Tansey and Y. Chen},
  title   = {Quantile regression with ReLU Networks: Estimators and minimax rates},
  journal = {Journal of Machine Learning Research},
  year    = {2022},
  volume  = {23},
  number  = {247},
  pages   = {1--42},
}

@incollection{10.1007/978-1-4612-2856-1_25,
  title={Nonparametric estimation of conditional quantiles using neural networks},
  author={H. White},
  booktitle={Computing Science and Statistics: Statistics of Many Parameters: Curves, Images, Spatial Models},
  pages={190--199},
  year={1992},
  publisher={Springer}
}

@article{10.5555/1248547.1248582,
  title={Quantile regression forests},
  author={N. Meinshausen},
  journal={Journal of machine learning research},
  volume={7},
  number={6},
  year={2006}
}

@INPROCEEDINGS{CL11,
	author={Cl\'emen\c{c}on, S.},
	year={2011},
	title={{On ${U}$-processes and clustering performance}},
	pages={37--45},
	booktitle={Advances in {N}eural {I}nformation {P}rocessing {S}ystems}
}

@book{BHS15,
  title={Metric Learning},
  author={Bellet, A. and Habrard, A. and Sebban, M.},
  series={Synthesis Lectures on Artificial Intelligence and Machine Learning},
  year={2015},
}

@article{JMLR:v17:14-265,
  author  = {A. K. Menon and R. C. Williamson},
  title   = {Bipartite Ranking: a Risk-Theoretic Perspective},
  journal = {Journal of Machine Learning Research},
  year    = {2016},
  volume  = {17},
  number  = {195},
  pages   = {1--102},
}

@book{delapena1999,
  title={Decoupling: from dependence to independence},
  author={V.H. {De la Pena} and E. Gin{\'e}},
  year={1999},
  publisher={Springer New York}
}

@article{JMLR:v7:takeuchi06a,
  author = {I. Takeuchi and Q. V. Le and T. D. Sears and A. J. Smola},
  journal = {Journal of Machine Learning Research},
  number = {45},
  pages = {1231-1264},
  title = {Nonparametric Quantile Estimation},
  volume = {7},
  year = {2006}
}

@article{koenker1978,
  author = {Koenker, R. and Bassett, G.},
  journal = {Econometrica},
  number = {1},
  pages = {33--50},
  publisher = {[Wiley, Econometric Society]},
  title = {Regression Quantiles},
  volume = {46},
  year = {1978}
}

@article{major2006,
  author = {P. Major},
  journal = {Probability Theory and Related Fields},
  pages = {489-537},
  title = {An estimate on the supremum of a nice class of stochastic integrals and U-processes},
  volume = {134},
  year = {2006}
}

@article{schoelkopf2000,
  title={New support vector algorithms},
  author={B. Sch{\"o}lkopf and A.J. Smola and R.C. Williamson and P.L. Bartlett},
  journal={Neural computation},
  year={2000},
  publisher={MIT Press}
}

@inproceedings{steinwart2007,
  author = {I. Steinwart and A. Christmann},
  booktitle = {Advances in Neural Information Processing Systems 20},
  pages = {305--312},
  title = {How SVMs can estimate quantiles and the median},
  year = {2007}
}

@book{steinwart2008support,
  title={Support vector machines},
  author={I. Steinwart and A. Christmann},
  year={2008},
  publisher={Springer Science \& Business Media}
}

@article{Steinwart_2011,
  author = {I. Steinwart and A. Christmann},
  journal = {Bernoulli},
  number = {1},
  publisher = {Bernoulli Society for Mathematical Statistics and Probability},
  title = {Estimating conditional quantiles with the help of the pinball loss},
  volume = {17},
  year = {2011}
}

@inproceedings{taktak12,
  title={Support vector machines for survival regression},
  author={Eleuteri, A. and Taktak, A. F. G.},
  booktitle={International Meeting on Computational Intelligence Methods for Bioinformatics and Biostatistics},
  pages={176--189},
  year={2011},
  organization={Springer}
}

@ARTICLE{Blom76,
	author={Blom, G.},
	year={1976},
	title={{Some properties of incomplete ${U}$-statistics}},
	journal={Biometrika},
	volume={63},
	number={3},
	pages={573--580}
}

@ARTICLE{Janson84,
	author={Janson, S.},
	year={1984},
	title={{The asymptotic distributions of Incomplete ${U}$-statistics}},
	journal={Z. Wahrsch. verw. Gebiete},
	volume={66},
	pages={495--505}
}

@inproceedings{lundberg_unified_2017,
	location = {Red Hook, {NY}, {USA}},
	title = {A unified approach to interpreting model predictions},
	series = {{NIPS}'17},
	pages = {4768--4777},
	booktitle = {Proceedings of the 31st International Conference on Neural Information Processing Systems},
	author = {Lundberg, S. M. and Lee, S.},
	urldate = {2026-01-15},
	year = {2017},
}

@BOOK{Ser80,
  author =       {R.J. Serfling},
  title =        {Approximation theorems of mathematical statistics},
  publisher =    {Wiley},
  year =         {1980},
}

@inproceedings{papa_sgd_2015,
	title = {{SGD} Algorithms based on Incomplete U-statistics: Large-Scale Minimization of Empirical Risk},
	volume = {28},
	shorttitle = {{SGD} Algorithms based on Incomplete U-statistics},
	booktitle = {Advances in Neural Information Processing Systems},
	author = {Papa, G. and Clémençon, S. and Bellet, A.},
	urldate = {2026-01-26},
	year = {2015},
}

@inproceedings{deng_arcface_2019,
	title = {{ArcFace}: Additive Angular Margin Loss for Deep Face Recognition},
	shorttitle = {{ArcFace}},
	eventtitle = {2019 {IEEE}/{CVF} Conference on Computer Vision and Pattern Recognition ({CVPR})},
	pages = {4685--4694},
	booktitle = {2019 {IEEE}/{CVF} Conference on Computer Vision and Pattern Recognition ({CVPR})},
	author = {Deng, J. and Guo, J. and Xue, N. and Zafeiriou, S.},
	urldate = {2026-01-27},
	year = {2019},
}

@article{Koenker01121999,
author = {R. Koenker and J. A. F. Machado},
title = {Goodness of Fit and Related Inference Processes for Quantile Regression},
journal = {Journal of the American Statistical Association},
volume = {94},
number = {448},
pages = {1296--1310},
year = {1999},
publisher = {Taylor \& Francis},
eprint = { 
    
    
        https://www.tandfonline.com/doi/pdf/10.1080/01621459.1999.10473882
}
}

@article{dutta_quality,
author = {Dutta, A. and Veldhuis, R. and Spreeuwers, L.},
year = {2015},
month = {10},
pages = {},
title = {Predicting Face Recognition Performance Using Image Quality},
journal = {IJCB 2014 - 2014 IEEE/IAPR International Joint Conference on Biometrics},
}

@techreport{Grother2022FRVT8,
  author       = {Grother, P.},
  title        = {Face Recognition Vendor Test (FRVT) Part 8: Summarizing Demographic Differentials},
  year         = {2022},
}

@article{lee_incomplete_1982,
	title = {On Incomplete U-Statistics Having Minimum Variance},
	volume = {24},
	pages = {275--282},
	number = {3},
	journal = {Australian Journal of Statistics},
	author = {Lee, Alan J.},
	urldate = {2026-02-03},
	year = {1982},
	langid = {english},
}

@article{JMLR:v17:15-012,
  author  = {S. Cl{{\'e}}men{\c{c}}on and I. Colin and A. Bellet},
  title   = {Scaling-up Empirical Risk Minimization: Optimization of Incomplete $U$-statistics},
  journal = {Journal of Machine Learning Research},
  year    = {2016},
  volume  = {17},
  number  = {76},
  pages   = {1--36},
}

@misc{bolme_data_2024-1,
	title = {From Data to Insights: A Covariate Analysis of the {IARPA} {BRIAR} Dataset for Multimodal Biometric Recognition Algorithms at Altitude and Range},
	shorttitle = {From Data to Insights},
	number = {{arXiv}:2409.01514},
	publisher = {{arXiv}},
	author = {Bolme, D. S. and Aykac, D. and Shivers, R. and Brogan, J. and Barber, N. and Zhang, B. and Davies, L. and Cornett, D.},
	urldate = {2026-02-04},
	year = {2024},
	eprinttype = {arxiv},
	eprint = {2409.01514 [cs]},
}

@book{massart_2007,
	location = {Berlin, Heidelberg},
	title = {Concentration Inequalities and Model Selection},
	volume = {1896},
	rights = {http://www.springer.com/tdm},
	series = {Lecture Notes in Mathematics},
	publisher = {Springer},
    author = {P. Massart},
	urldate = {2026-02-15},
	year = {2007},
	langid = {english},
}

@article{10.1214/009053606000000786,
author = {P. Massart and {\'E}. N{\'e}d{\'e}lec},
title = {{Risk bounds for statistical learning}},
volume = {34},
journal = {The Annals of Statistics},
number = {5},
publisher = {Institute of Mathematical Statistics},
pages = {2326 -- 2366},
year = {2006},
}

@article{10.1214/aos/1079120131,
author = {A. B. Tsybakov},
title = {{Optimal aggregation of classifiers in statistical learning}},
volume = {32},
journal = {The Annals of Statistics},
number = {1},
publisher = {Institute of Mathematical Statistics},
pages = {135 -- 166},
year = {2004},
}

@inproceedings{wang_cosface_2018,
	location = {Salt Lake City, {UT}},
	title = {{CosFace}: Large Margin Cosine Loss for Deep Face Recognition},
	shorttitle = {{CosFace}},
	eventtitle = {2018 {IEEE}/{CVF} Conference on Computer Vision and Pattern Recognition ({CVPR})},
	pages = {5265--5274},
	booktitle = {2018 {IEEE}/{CVF} Conference on Computer Vision and Pattern Recognition},
	publisher = {{IEEE}},
	author = {Wang, H. and Wang, Y. and Zhou, Z. and Ji, X. and Gong, D. and Zhou, J. and Li, Z. and Liu, W.},
	urldate = {2026-02-16},
	year = {2018},
	langid = {english},
}

@inproceedings{sangnier_joint_2016,
	title = {Joint quantile regression in vector-valued {RKHSs}},
	volume = {29},
	booktitle = {Advances in Neural Information Processing Systems},
	publisher = {Curran Associates, Inc.},
	author = {Sangnier, M. and Fercoq, O. and d' Alché-Buc, F.},
	urldate = {2026-02-20},
	year = {2016},
}

@article{knight_limiting_1998,
	title = {Limiting Distributions for L1 Regression Estimators under General Conditions},
	volume = {26},
	pages = {755--770},
	number = {2},
	journal = {The Annals of Statistics},
	publisher = {Institute of Mathematical Statistics},
	author = {Knight, K.},
	urldate = {2026-02-23},
	year = {1998},
}

@article{boucheron_theory_2005,
	title = {Theory of Classification: A Survey of Some Recent Advances},
	shorttitle = {Theory of Classification},
	journaltitle = {{ESAIM}: Probability and Statistics, v.9, 323-375 (2005)},
	shortjournal = {{ESAIM}: Probability and Statistics, v.9, 323-375 (2005)},
	author = {Boucheron, S. and Bousquet, O. and Lugosi, G.},
	year = {2005},
}

@article{HAUSSLER1995217,
title = {Sphere packing numbers for subsets of the Boolean n-cube with bounded Vapnik-Chervonenkis dimension},
journal = {Journal of Combinatorial Theory, Series A},
volume = {69},
number = {2},
pages = {217-232},
year = {1995},
issn = {0097-3165},
doi = {https://doi.org/10.1016/0097-3165(95)90052-7},
url = {https://www.sciencedirect.com/science/article/pii/0097316595900527},
author = {David Haussler}
}

\onecolumn
\appendix

\section{Technical Proofs}
\subsection{Proof of Proposition \ref{prop:alpha_condition}}
\label{proof:alpha_condition}

\begin{proposition*}
Suppose that Assumption~\ref{ass:density} is fulfilled. For all $q \in \mathcal Q$, we have:
$$
\Var(k_q(V)) \;\le\; C_{\mathrm{var}} \, \mathcal{E}(q),
$$
where $C_{\mathrm{var}} = 2M_\tau^2/\nu$ and $M_\tau = \max(\tau, 1-\tau)$.
\end{proposition*}

\begin{proof}
This proof establishes the \textit{variance-excess risk relation} by showing that the variance of the first-order projection is controlled by the excess risk under a non-vanishing density condition.

\paragraph{Bounding the Variance by the $L_2$ Distance.}
Recall the first-order projection $k_q(v) = \mathbb{E}[K_{q}(v,V')]$. Since the variance is bounded by the second moment, $\text{Var}(k_q(V)) \leq \mathbb{E}[k_q(V)^2]$. By Jensen's inequality and the definition of $K_q(v,v') = H_q(v,v') - H_{Q_s(\tau \mid \cdot)}(v,v')$, we have:
\begin{equation*}
\mathbb{E}[k_q(V)^2] = \mathbb{E}\left[ \left( \mathbb{E}[H_q(V,V') - H_{Q_s(\tau \mid \cdot)}(V,V') \mid V] \right)^2 \right] \leq \mathbb{E}[K_q(V,V')^2].
\end{equation*}
The pinball loss $\rho_\tau(u) = u(\tau - \mathbf{1}\{u<0\})$ is Lipschitz continuous with constant $M_\tau = \max(\tau, 1-\tau)$. Thus, for any $S, q_1, q_2$, we have the pointwise inequality $|\rho_\tau(S - q_1) - \rho_\tau(S - q_2)| \leq M_\tau |q_1 - q_2|$. Applying this to the kernel $K_q(V,V') = \rho_{\tau}(s(X,X') - q(Z,Z')) - \rho_{\tau}(s(X,X') - Q_s(\tau \mid Z,Z'))$, we obtain:
\begin{equation}
\label{eq:var_bound_final}
\text{Var}(k_q(V)) \leq M_\tau^2 \mathbb{E}[(q(Z,Z') - Q_s(\tau \mid Z,Z'))^2] = M_\tau^2 \|q - Q_s(\tau \mid \cdot)\|_{L_2(P)}^2,
\end{equation}
where $\|f\|_{L_2(P)}^2 := \mathbb{E}[f(Z,Z')^2]$ denotes the squared $L_2(P)$ norm with respect to the joint distribution $P$ of $(Z,Z')$.

\paragraph{Lower-bounding the Excess Risk.}
We analyze the point-wise difference in the loss. Let $u = s(X,X') - Q_s(\tau \mid Z,Z')$ and $\Delta q = q(Z,Z') - Q_s(\tau \mid Z,Z')$. Using the identity for the difference of pinball losses \citep{knight_limiting_1998}:
\[
\rho_\tau(u - \Delta q) - \rho_\tau(u) = -\Delta q (\tau - \mathbf{1}\{u < 0\}) + \int_0^{\Delta q} (\mathbf{1}\{u < t\} - \mathbf{1}\{u < 0\}) dt.
\]
Taking the conditional expectation with respect to $s$ given $(Z, Z')$, the first term vanishes by the definition of the conditional $\tau$-quantile. Under Assumption \ref{ass:density}, the conditional density $p_{s|Z,Z'}$ is bounded below by $\nu > 0$ in the neighborhood of the quantile. By a first-order Taylor expansion of the conditional CDF:
\[
\mathbb{E}_S[\rho_\tau(s - q) - \rho_\tau(s - Q_s(\tau \mid \cdot)) \mid Z, Z'] = \int_0^{\Delta q} (F_{S|Z,Z'}(Q_s(\tau \mid \cdot) + t) - F_{S|Z,Z'}(Q_s(\tau \mid \cdot))) dt \geq \frac{\nu}{2} (\Delta q)^2.
\]
Integrating over $(Z, Z')$ yields:
\begin{equation}
\label{eq:risk_lower_final}
\mathcal{E}(q) \geq \frac{\nu}{2} \|q - Q_s(\tau \mid \cdot)\|_{L_2(P)}^2.
\end{equation}

\paragraph{Conclusion.}
Combining \eqref{eq:var_bound_final} and \eqref{eq:risk_lower_final}, we have:
\[
\text{Var}(k_q(V)) \leq M_\tau^2 \|q - Q_s(\tau \mid \cdot)\|_{L_2(P)}^2 \leq \frac{2 M_\tau^2}{\nu} \mathcal{E}(q).
\]
This confirms the relation $\text{Var}(k_q(V)) \leq C_{\mathrm{var}} \mathcal{E}(q)$ with $C_{\mathrm{var}} = 2M_\tau^2/\nu$.
\end{proof}

\subsection{Proof of Theorem \ref{thm:pinball-uerm}}
\label{app:proof-thm1}

\begin{theorem*}
Suppose that Assumptions \ref{hyp:Q} and \ref{ass:density} are fulfilled. Assume also that $Q_s(\tau\mid \cdot)\in \mathcal{Q}$.
Then there exists a universal constant $C>0$ such that for all $\delta\in(0,1)$, we have with probability at least $1-\delta$,
\begin{equation}\label{eq:bound1}
\mathcal E(\hat q_\tau) \leq     C \left(
\frac{L\log(n)}{n}
+
\frac{(B+\vert\vert s\vert\vert_{\infty})\log(1/\delta) + L}{n}\right).
\end{equation}
\end{theorem*}

\begin{proof}
We analyze the excess risk by first decomposing the second-order $U$-statistic using the Hoeffding decomposition \citep{hoeffding1948}. This allow us to write the $U$-statistic as a sum of a linear empirical process and a degenerate remainder:
\[
\widehat{\mathcal{E}}_n(q) - \mathcal{E}(q) = 2 T_n(q) + W_n(q),
\]
where $T_n(q)$ is the first-order projection:
\[
T_n(q) = \frac{1}{n} \sum_{i=1}^n k_q(V_i) - \mathcal{E}(q), \quad \text{with} \quad k_q(V) = \mathbb{E}[K_{q}(V, V') \mid V],
\]
and $W_n(q)$ is the degenerate $U$-statistic defined by the kernel $\hat{k}_q$:
\[
W_n(q) = \frac{1}{n(n-1)} \sum_{i \neq j} \hat{k}_q(V_i, V_j),
\]
where the canonical (degenerate) kernel is given by
\begin{equation}
\label{eq:h_degenerate}
    \hat{k}_q(V,V') = K_{q}(V, V') - k_{q}(V) - k_{q}(V') - \mathcal{E}(q).
\end{equation}
By construction, the functions $k_q(V_i) - \mathcal{E}(q)$ are independent, centered ($\mathbb{E}[k_q(V) - \mathcal{E}(q)] = 0$), and, under Assumption \ref{hyp:Q} and the Lipschitz property of the pinball loss, uniformly bounded.

\paragraph{Control of the Linear Term.} 
To bound $\sup_{q \in \mathcal{Q}} |T_n(q)|$, we employ the local complexity framework of \cite{clemencon_empirical} and  \cite{massart_2007}. Define the $L_2(P)$ pseudo-distance $d(q, q') = (\mathbb{E}[(k_q(V) - k_{q'}(V))^2])^{1/2}$. Let $\phi$ be a non-decreasing function such that $\phi(x)/x$ is non-increasing and $\phi(1)=1$ and such that for all $q \in \mathcal{Q}$, 
\begin{equation}
\label{eq:phi_modulus}
    \sqrt{n} \mathbb{E} \left[ \sup_{q' \in \mathcal{Q}, d(q, q') \leq u} \left| T_n(q) - T_n(q') \right| \right] \leq \phi(u).
\end{equation}
Under Assumption \ref{hyp:Q}, $\mathcal{Q}$ has VC-dimension $L$. By Lemma $6.5$ in \cite{massart_2007}, we characterize the local entropy and set:
\[
\phi(u) = Ku \sqrt{L \left( 1 + \log(\min \{u^{-1}, 1\}) \right)},
\]
where $K$ is a universal constant. Following Assumption \ref{ass:density} and using Proposition \ref{prop:alpha_condition}, we have a variance-risk link of the form $w(\varepsilon) = \sqrt{C_{\mathrm{var}}} \varepsilon$. The critical radius $\varepsilon_*$ is the solution to the fixed-point equation $\sqrt{n}\varepsilon^2 = \phi(w(\varepsilon))$, which yields:
\[
\varepsilon_*^2 = C_0 \left( \frac{L \log n}{n} \right).
\]
Applying Massart's concentration theorem for empirical processes (Theorem 8.3 \cite{massart_2007}), for any $\delta \in (0,1)$, there exists $C_1 > 0$ such that with probability at least $1 - \delta/2$:
\begin{equation}
\label{eq:bound_tn}
\sup_{q \in \mathcal{Q}} |T_n(q)| \leq C_1 \left( \left( \frac{L \log n}{n} \right) + \frac{(B+\vert\vert s\vert\vert_{\infty}) \log(2/\delta)}{n} \right).
\end{equation}

\paragraph{Control of the Degenerate Term.} 
The remainder $W_n(q)$ is a degenerate $U$-statistic of order 2. Since the class $\mathcal{Q}$ has VC-dimension $L$, the class of kernels $\mathcal{K} = \{K_q : q \in \mathcal{Q}\}$ also possesses a controlled VC-dimension proportional to $L$. By Theorem 5 in \cite{clemencon_empirical}, the following supremum bound holds with probability at least $1 - \delta/2$:
\begin{equation}
\label{eq:bound_wn}
\sup_{q \in \mathcal{Q}} |W_n(q)| \leq C_2 \frac{L + \log(2/\delta)}{n}.
\end{equation}
The constant $C_2$ depends on the quantile regression loss, in particular on the quantile $\tau$ and on $B$ from Assumption \ref{ass:distribution} such that $\mid q(z, z') \leq B $ for all $z$, $z'$ in $\mathcal Z$ and any $q$ in $\mathcal{Q}$. $C_2$ also depends on the bound of the covering number from \cite{HAUSSLER1995217}, see Theorem 1 therein.

\paragraph{Conclusion.} 
The excess risk of the ERM estimator $\hat{q}_\tau$ satisfies $\mathcal{E}(\hat{q}_\tau) \leq 2 (\sup_{q \in \mathcal{Q}} |T_n(q)| + \sup_{q \in \mathcal{Q}} |W_n(q)|)$. Combining \eqref{eq:bound_tn} and \eqref{eq:bound_wn} via a union bound, we obtain the fast rate (up to logarithm factors) with probability at least $1 - \delta$,
\[
\mathcal{E}(\hat{q}_\tau) = O_p \left( \frac{L\log(n)}{n} + \frac{(B+\vert\vert s\vert\vert_{\infty}) \log(1/\delta) + L}{n} \right).
\]

\end{proof}


\subsection{Proof of Corollary \ref{corr:l2_distance}}
\label{app:proof-corollary-l2}
\begin{corollary*}
Suppose that Assumptions \ref{hyp:Q} and \ref{ass:density} are fulfilled. Assume also that $Q_s(\tau\mid \cdot)\in \mathcal{Q}$.
Then there exists a universal constant $C>0$ such that for all $\delta\in(0,1)$, we have with probability at least $1-\delta$,
\begin{multline}\label{eq:bound2}
    \|\hat{q}_\tau - Q_s(\tau \mid \cdot)\|_{L_2(P)} \leq \\ \sqrt{\frac{2C}{\nu}\Bigg(\frac{L\log(n)}{n} + \frac{(B+\vert\vert s\vert\vert_{\infty}) \log(1/\delta) + L}{n}\Bigg)}.
\end{multline}

\end{corollary*}
\begin{proof}
The result follows directly from the lower bound on the excess risk derived in the proof of Proposition \ref{prop:alpha_condition} and by applying Theorem \ref{thm:pinball-uerm}. Specifically, as established in Equation \eqref{eq:risk_lower_final}, 
\[
\|q - Q_s(\tau \mid \cdot)\|_{L_2(P)}^2 \leq \frac{2}{\nu} \mathcal{E}(q).
\]
Thus, applying Theorem  \ref{thm:pinball-uerm}, we obtain the desired results:
\begin{equation*}
    \|\hat{q}_\tau - Q_s(\tau \mid \cdot)\|_{L_2(P)} \leq \sqrt{\frac{2C}{\nu}\Bigg(\frac{L\log(n)}{n} + \frac{(B+\vert\vert s\vert\vert_{\infty}) \log(1/\delta) + L}{n}\Bigg)}
\end{equation*}

\end{proof}
This inequality demonstrates that under a non-vanishing density at the quantile, the excess risk provides strong control over the $L_2$ deviation from the target function, justifying its use in the variance--excess risk relationship for fast rate analysis.



\section{Discussion on the Assumptions}
\label{app:ass-discussion}
In this section, we revisit the main assumptions underlying our theoretical results and discuss possible relaxations. We first revisit the uniform boundedness of the function class $\mathcal{Q}$ which is part of Assumption \ref{hyp:Q}, and show that it can be replaced by a sub-Gaussian tail condition. We then examine the assumption that the conditional quantile function $Q_s(\tau \mid Z, Z')$ belongs to $\mathcal{Q}$ and characterize the approximation error incurred when this requirement is relaxed. We also study the model selection case when we have a countable collection of hypothesis classes $\{\mathcal{Q}_k\}_{k \ge 1}$. Together, these results extend the applicability of our theory to broader function classes and more general data generating processes.
\subsection{Relaxing the Uniform Boundedness}
\label{app:relax_A1}

We assumed in Assumption \ref{hyp:Q} that there exists $B>0$ such that $|q(z,z')| \le B$
for all $z$, $z'$ in $\mathcal Z$ and any $q$ in $\mathcal{Q}$.

We replace this deterministic boundedness condition with a stochastic tail assumption.

\begin{assumption}[Sub-Gaussian envelope]
\label{ass:relaxed_boundedness}
For all $q \in \mathcal Q$, the centered random variable 
$q(Z,Z') - \mathbb E[q(Z,Z')]$ is sub-Gaussian with parameter $\sigma_q^2$ (not necessarily its variance). Moreover, there exist constants $M_1, M_2 \ge 0$ such that
\[
\sup_{q \in \mathcal Q} |\mathbb E[q(Z,Z')]| \le M_1,
\qquad
\sup_{q \in \mathcal Q} \sigma_q^2 \le M_2.
\]
\end{assumption}

Under Assumption \ref{ass:relaxed_boundedness}, for any $t \in \mathbb R$,
\[
\mathbb E\!\left[
e^{t (q(Z,Z') - \mathbb E[q(Z,Z')])}
\right]
\le 
\exp\!\left(\frac{\sigma_q^2 t^2}{2}\right),
\]
which implies the tail bound
\[
\mathbb P\!\left(
|q(Z,Z') - \mathbb E[q(Z,Z')]| \ge u
\right)
\le 
2 \exp\!\left(-\frac{u^2}{2\sigma_q^2}\right),
\quad \forall u \ge 0.
\]
Hence, for any $\delta \in (0,1)$, with probability at least $1-\delta$,
\[
|q(Z,Z') - \mathbb E[q(Z,Z')]| 
\le 
\sigma_q \sqrt{2 \log \frac{2}{\delta}}.
\]

Combining this concentration with Dudley’s entropy integral yields that, with probability at least $1-\delta$,
\[
\sup_{q \in \mathcal Q}
\left| \frac{1}{n(n-1)}
\sum_{i<j} q(Z_i, Z_j)
\right|
\le 
B_\delta
\]
where $(Z_i)_{i=1}^n$ are the covariate components of the i.i.d.\ sample $(V_i)_{i=1}^n$, $C>0$ is a universal constant and 
\begin{equation*}
   B_\delta =  M_1  + C\sqrt{M_2 L} + \sqrt{M_2}\sqrt{2\log \frac{2}{\delta}}.
\end{equation*}

This provides a high-probability control that replaces the deterministic uniform boundedness assumption. The resulting bound depends only logarithmically on $1/\delta$ and can be incorporated into the proof of Theorem~\ref{thm:pinball-uerm}.

\begin{lemma}
Assume that Assumptions \ref{ass:density} and \ref{ass:relaxed_boundedness} holds and that the class $\mathcal{Q}$ is a bounded VC class with VC dimension $L$. Then there exists a constant $C>0$ such that for all $\delta \in (0,1)$, with probability at least $1-2\delta$,
\[
\mathcal E(\hat q_\tau) \leq     C \left(
\frac{L\log(n)}{n}
+
\frac{(B_\delta+\vert\vert s\vert\vert_{\infty})\log(1/\delta) + L}{n}\right).
\]
\end{lemma}

\begin{proof}
The argument follows that of Appendix~\ref{app:proof-thm1}. 
The uniform boundedness condition is replaced by the high-probability envelope obtained above, and the remainder of the proof proceeds identically on this event.
\end{proof}

Consequently, Theorem~\ref{thm:pinball-uerm} extends to sub-Gaussian function classes and more general data distributions.

\subsection{Relaxing the Realizability Assumption: Approximation Error}
\label{app:approximation_error}

In Theorem \ref{thm:pinball-uerm}, we operate under the realizability assumption that the true conditional $\tau$-quantile function $Q_{s}(\tau \mid Z, Z')$ belongs to the function class $\mathcal{Q}$. In many practical scenarios, $\mathcal{Q}$ may be misspecified, and $Q_{s}(\tau \mid Z, Z') \notin \mathcal{Q}$. In such cases, we define the best-in-class predictor (the oracle) as: $\bar{q}_\tau = \arg\min_{q \in \mathcal{Q}} R(q).$
We then decompose the total excess risk $\mathcal{E}(\hat{q}_\tau) = R(\hat{q}_\tau) - R(q^*_\tau)$ into an estimation error component and an approximation error component:
\begin{align*}
    \mathcal{E}(\hat{q}_\tau) &= \underbrace{R(\hat{q}_\tau) - R(\bar{q}_\tau)}_{\text{Estimation Error}} + \underbrace{R(\bar{q}_\tau) - R(q^*_\tau)}_{\text{Approximation Error}}.
\end{align*}

The first term is bounded by the stochastic fluctuations of the $U$-process over $\mathcal{Q}$, while the second term, which we denote as the \textit{approximation error} $\mathcal{A}(\mathcal{Q}) = \inf_{q \in \mathcal{Q}} \{ R(q) - R(q^*_\tau) \}$, depends solely on the capacity of the class $\mathcal{Q}$ to represent the true quantile function.

\begin{proposition}[General Excess Risk Bound]
Under Assumptions \ref{hyp:Q} and \ref{ass:density}, the excess risk of the pairwise empirical risk minimizer $\hat{q}_\tau$ satisfies, with probability at least $1-\delta$:
\begin{equation*}
    \mathcal{E}(\hat{q}_\tau) \leq C \left(\frac{L \log n}{n} + \frac{(B+1)\log(n/\delta) + L}{n} \right) + \mathcal{A}(\mathcal{Q}).
\end{equation*}
\end{proposition}

\begin{proof}
The proof follows by applying the concentration arguments of Theorem \ref{thm:pinball-uerm} and accounting for the constant bias term $\mathcal{A}(\mathcal{Q})$ yields the result.
\end{proof}

While it is theoretically possible to improve this bound using Talagrand's inequality, such an analysis involves significant technical complexities regarding pairwise empirical processes. Given the scope of this work, we leave this extension for future research, noting that the foundational arguments for classification in \citep{boucheron_theory_2005} provide a roadmap for this development.



\subsection{Model Selection and Oracle Inequalities}
\label{app:model_selection}

We consider the selection of a quantile regression function from a countable collection of hypothesis classes $\{\mathcal{Q}_k\}_{k \ge 1}$. Each class $\mathcal{Q}_k$ possesses a VC dimension $L_k$. To identify the model that optimally balances approximation error and estimation complexity, we employ a penalized empirical risk minimization framework.

We define the model selection index $\hat{k}$ as:
\begin{equation}
\label{eq:srm_objective_updated}
\hat{k} = \arg\min_{k \ge 1} \left\{ \inf_{q \in \mathcal{Q}_k} \widehat{\mathcal{E}}_n(q) + \mathrm{pen}(n,k) \right\},
\end{equation}
where $\widehat{\mathcal{E}}_n(q)$ is the pairwise empirical risk. Given the concentration properties of the pinball loss $U$-statistic, we consider penalties of the form:
\begin{equation*}
\mathrm{pen}(n,k) = C \left( \frac{L_k\log(n) + \log(k)}{n} \right).
\end{equation*}

\begin{theorem}[Oracle Inequality]
\label{thm:oracle_inequality_updated}
Let $\hat{q}_{\hat{k}}$ be the estimator selected by \eqref{eq:srm_objective_updated}. Under the variance--excess risk relation with $\alpha=1$, there exists a constant $C > 0$ such that for any $\delta \in (0,1)$, with probability at least $1-\delta$:
\begin{equation*}
\mathcal{E}(\hat{q}_{\hat{k}}) \le \min_{k \ge 1} \left\{ \inf_{q \in \mathcal{Q}_k} \mathcal{E}(q) + 2\,\mathrm{pen}(n,k) \right\} + \frac{C \log(1/\delta)}{n}.
\end{equation*}
\end{theorem}

\begin{proof}
The proof proceeds via a weighted union bound over the model collection.

\paragraph{Localized Union Bound.}
From Theorem \ref{thm:pinball-uerm}, for any fixed class $\mathcal{Q}_k$, the uniform deviation is bounded by $C(L_k + \log(n) + \log(1/\delta_k))/n$. We define weights $w_k > 0$ such that $\sum_{k=1}^\infty w_k = 1$ (e.g., $w_k = 6/(\pi^2 k^2)$) and set $\delta_k = \delta w_k$. 
By the union bound, with probability at least $1-\delta$, the following holds for all $k \ge 1$:
\begin{equation}
\label{eq:dev_n_log}
\sup_{q \in \mathcal{Q}_k} |\widehat{\mathcal{E}}_n(q) - \mathcal{E}(q)| \le C \left( \frac{L_k\log(n) + \log(1/w_k) + \log(1/\delta) + L_k}{n} \right).
\end{equation}
Recognizing that $\log(1/w_k) = C_1 \log k$ for an appropriate constant $C_1>0$, we define $\mathrm{pen}(n,k)$ to dominate the $k$ and $n$ dependent terms. Thus, on the high-probability event:
\[ \forall k \ge 1, \forall q \in \mathcal{Q}_k: \quad |\widehat{\mathcal{E}}_n(q) - \mathcal{E}(q)| \le \mathrm{pen}(n,k) + \frac{C \log(1/\delta)}{n}. \]

\paragraph{Comparison Logic.}
Applying the concentration bound to the selected model $\hat{k}$:
\[ \mathcal{E}(\hat{q}_{\hat{k}}) \le \widehat{\mathcal{E}}_n(\hat{q}_{\hat{k}}) + \mathrm{pen}(n, \hat{k}) + \frac{C \log(1/\delta)}{n}. \]
By the definition of the minimization in \eqref{eq:srm_objective_updated}, for any $k \ge 1$:
\[ \widehat{\mathcal{E}}_n(\hat{q}_{\hat{k}}) + \mathrm{pen}(n, \hat{k}) \le \widehat{\mathcal{E}}_n(\hat{q}_k) + \mathrm{pen}(n, k) \le \widehat{\mathcal{E}}_n(q^*_\tau) + \mathrm{pen}(n, k). \] Finally, applying the concentration bound \eqref{eq:dev_n_log} to $\widehat{\mathcal{E}}_n(q^*_\tau)$ yields:
\[ \widehat{\mathcal{E}}_n(q^*_\tau) \le \mathcal{E}(q^*_\tau) + \mathrm{pen}(n, k) + \frac{C \log(1/\delta)}{n}. \]

\paragraph{Conclusion.}
Combining the inequalities, 
\[ \mathcal{E}(\hat{q}_{\hat{k}}) \le \mathcal{E}(q^*_\tau) + 2\mathrm{pen}(n, k) + \frac{2C \log(1/\delta)}{n}, \]
which finishes the proof.
\end{proof}
This oracle inequality demonstrates that the proposed selection rule achieves a nearly optimal balance between approximation error and estimation complexity across the model hierarchy $\{\mathcal{Q}_k\}_{k\geq 1}$. While the empirical implementation and calibration of the penalty constants are left for future research, these results establish a rigorous theoretical framework for adaptive pairwise learning under the density condition.

\section{Additional Experiments}
This section reports additional experiments designed to further analyze the behavior and practical relevance of pairwise quantile regression. 
We begin with the synthetic setup that allows precise control over the data-generating mechanism, enabling a detailed study of loss behavior, model comparison, and computational trade-offs. 
We then extend the analysis to facial recognition data, where we assess empirical coverage, performance at extreme quantiles, and feature influence in a realistic setting.

\subsection{Synthetic Experiments}
\label{app:synthetic}

This appendix provides additional technical details and experimental settings for the synthetic simulations discussed in Section~\ref{sec:synthetic}.

\paragraph{Implementation Details.} 
All experiments were implemented in PyTorch using a fixed global random seed of $42$ to ensure reproducibility across data generation, weight initialization, and optimization. The synthetic dataset comprises $n_{\text{train}}=1000$ training samples ($499,500$ pairs) and $n_{\text{test}}=200$ evaluation samples ($19,000$ pairs). For our primary model, we employ a Multilayer Perceptron (MLP) with two hidden layers of $32$ units each and ReLU activations. The model is trained for $200$ epochs using the Adam optimizer with a learning rate of $0.01$ and a batch size of $64$. In Figure \ref{fig:Model_Comparison}, we compare this approach against LightGBM and Gradient Boosting Regressor baselines, utilizing their respective scikit-learn interfaces with tuned hyperparameters.

\paragraph{Ground Truth Estimation.} 
Since the analytical form of the conditional quantile $Q_s(\tau \mid Z, Z')$ is often intractable, we estimate the ground truth via Monte Carlo sampling. For each pair $(Z_i, Z_j)$, we draw $M=10,000$ independent realizations from the latent distributions $\mathcal{N}(\mu_{Z_i}, \sigma_{Z_i})$ and $\mathcal{N}(\mu_{Z_j}, \sigma_{Z_j})$. The empirical $\tau$-quantile of the resulting scores serves as the proxy for the true $Q_s(\tau \mid \cdot)$. We evaluate model performance using the Mean Absolute Error (MAE) relative to these Monte Carlo estimates:
\begin{equation}
    \text{MAE}(\mathcal{Q}) = \frac{1}{N_{\text{pairs}}} \sum_{i \neq j} \left| Q_s(\tau \mid Z_i, Z_j) - \hat{q}(Z_i, Z_j) \right|,
\end{equation}
where $N_{\text{pairs}}$ denotes the total number of evaluated pairs in the test set and $\hat{q}$ is the empirical minimizer over the class of functions $\mathcal{Q}$.

\begin{figure}[h]
    \centering
    \begin{minipage}{0.58\textwidth}
        \centering
        \includegraphics[width=\linewidth]{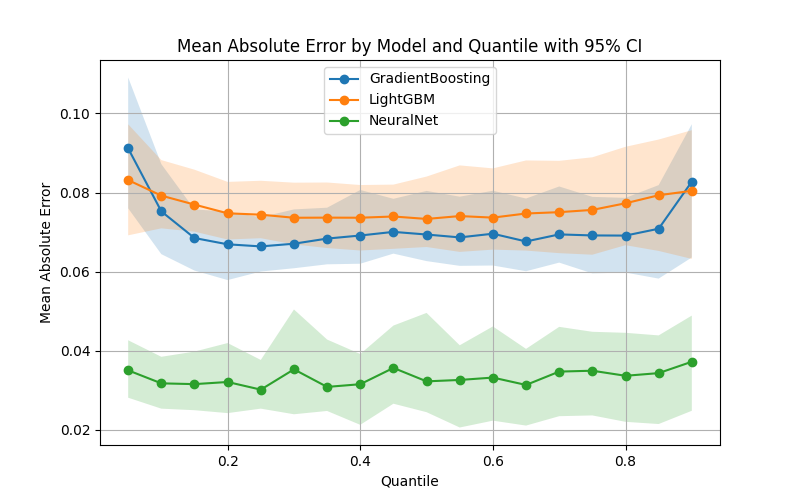}
        \caption{Model Comparison}
        \label{fig:Model_Comparison}
    \end{minipage}
\end{figure}
\paragraph{Incomplete $U$-statistics.} 
While the experiments in Section~\ref{sec:synthetic} utilize all possible pairwise combinations, resulting in a quadratic runtime complexity of $\mathcal{O}(n^2)$, this may become prohibitive for larger datasets. To address this, we investigate the trade-off between predictive performance and computational efficiency using incomplete U-statistics.

In this setting, we maintain $n=1000$ samples but subsample the total number of pairs. Specifically, we compare the complete case ($499,500$ pairs) against an incomplete case where only $6,907$ pairs are randomly sampled (approximately $n \log n$). Both models are trained for $100$ epochs and evaluated on a fixed test set of $4,995$ pairs. The results are summarized in Table~\ref{tab:complete_incomplete_comparison}. The transition from complete to incomplete $U$-statistics yields an order-of-magnitude reduction in training time with only a negligible decay in accuracy. This empirical evidence suggests that the statistical benefits of using the full quadratic set of pairs diminish quickly, confirming that our proposed framework remains scalable and practically applicable to large scale ranking and regression problems.

\begin{table}[ht]
    \centering
    \begin{tabular}{lrr}
        \toprule
        Model Approach & Runtime & MAE \\
        \midrule
        Complete ($n^2$) & 2min 13s & $9.4 \times 10^{-2}$ \\
        Incomplete ($n \log n$) & \textbf{2.6s} & $9.9 \times 10^{-2}$ \\
        \bottomrule
    \end{tabular}
    \caption{Computational and predictive comparison between complete and incomplete U-statistics.}
    \label{tab:complete_incomplete_comparison}
\end{table}

\subsection{Facial Recognition}
\label{app:facial_recognition}
\paragraph{Feature Impact.}\label{app:feature_impact}
Figures \ref{fig:shap_values_pos_0.01} and \ref{fig:shap_values_neg_0.99} show that feature importance differs markedly between impostor and genuine pairs. This behavior is expected, since the factors that cause impostor pairs to appear similar are not the same as those that lead genuine pairs to appear dissimilar.
For genuine pairs at quantile $0.01$, age difference and image quality score emerge as the dominant explanatory features. In addition, hair color becomes influential: pairs with matching hair color tend to receive higher similarity scores, whereas mismatched hair color is associated with lower scores. 

\begin{figure}[h!]
    \centering
    \begin{minipage}[b]{0.45\textwidth}
        \includegraphics[width=\textwidth]{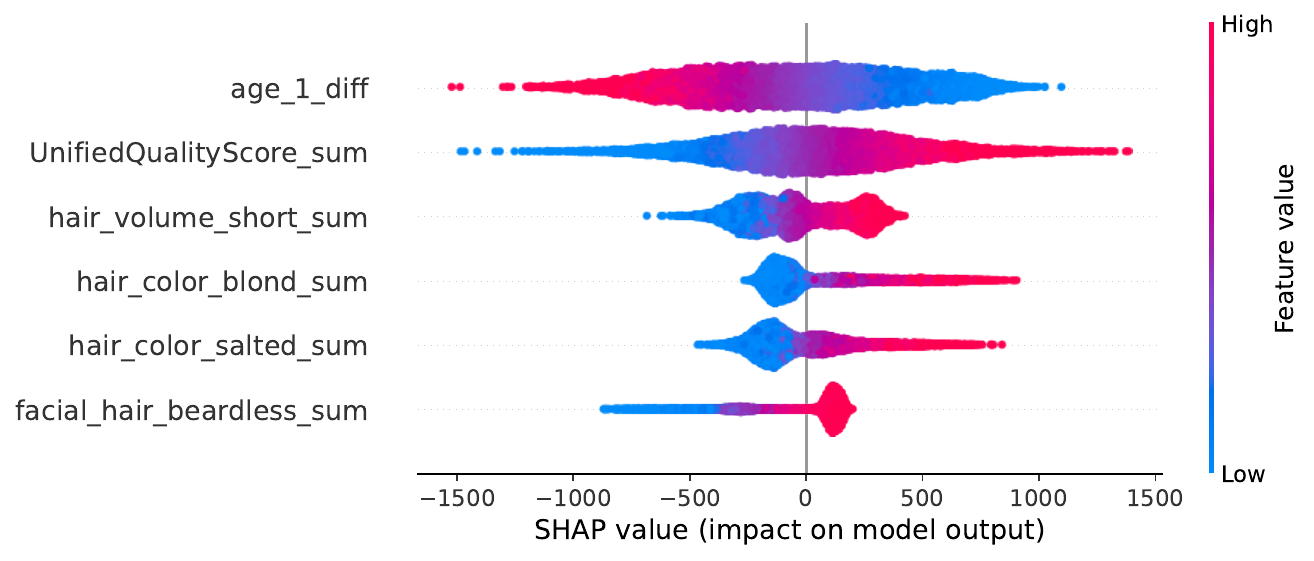}
        \caption{Shapley values for six of the most influencing features at the 1st percentile ($\tau = 0.01$) for genuine pairs. Larger age differences reduce similarity scores, while higher image quality increases them.}
        \label{fig:shap_values_pos_0.01}
    \end{minipage}
    \hfill
    \begin{minipage}[b]{0.45\textwidth}
        \centering
        \includegraphics[width=\textwidth]{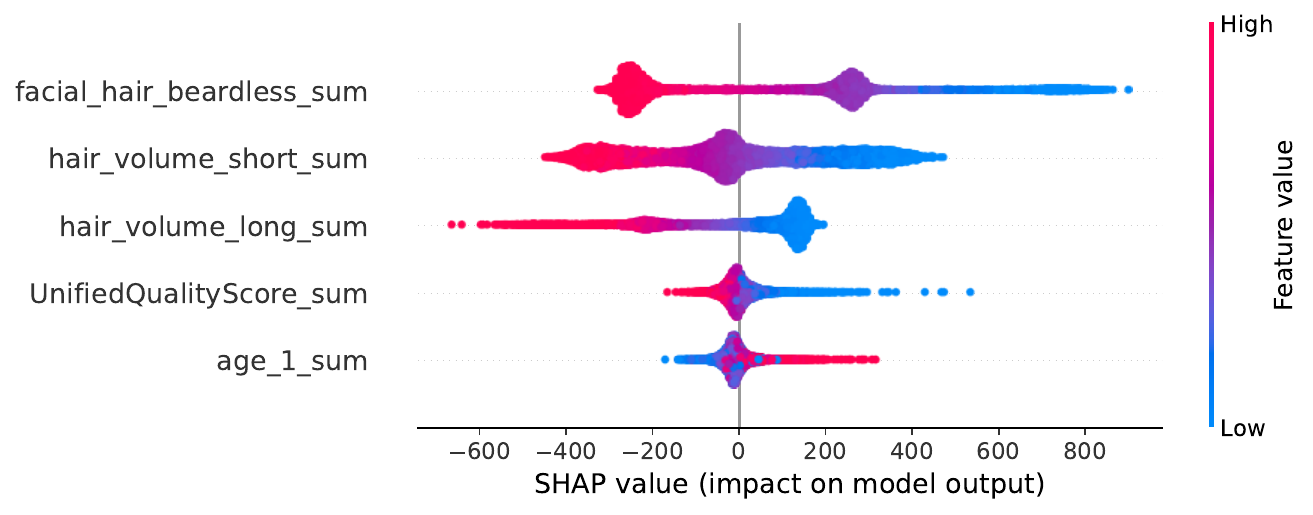}
        \caption{Shapley values for six of the most influencing features at the 99th percentile ($\tau = 0.99$) for impostor pairs. Lower-quality images increase predicted similarity, while greater differences in hair length reduce it.}
        \label{fig:shap_values_neg_0.99}
    \end{minipage}
\end{figure}

We further illustrate how individual features affect predicted similarity scores across the score range at different quantile levels for genuine pairs (Figures~\ref{fig:shap_age}–\ref{fig:shap_quality}). These plots show that a feature’s influence can vary across quantiles; for instance, the age difference has a stronger effect on increasing similarity than on decreasing it at $\tau = 0.1$ and at $\tau=0.5$.
Notably, the curves are almost monotonically increasing, indicating that low similarity scores are consistently negatively impacted by the considered features. Differences across features primarily lie in the score level at which a given feature begins to have a positive effect, which varies across quantiles.

\begin{figure}[h!]
    \centering
    \begin{minipage}{0.30\linewidth}
        \centering
        \includegraphics[width=\linewidth]{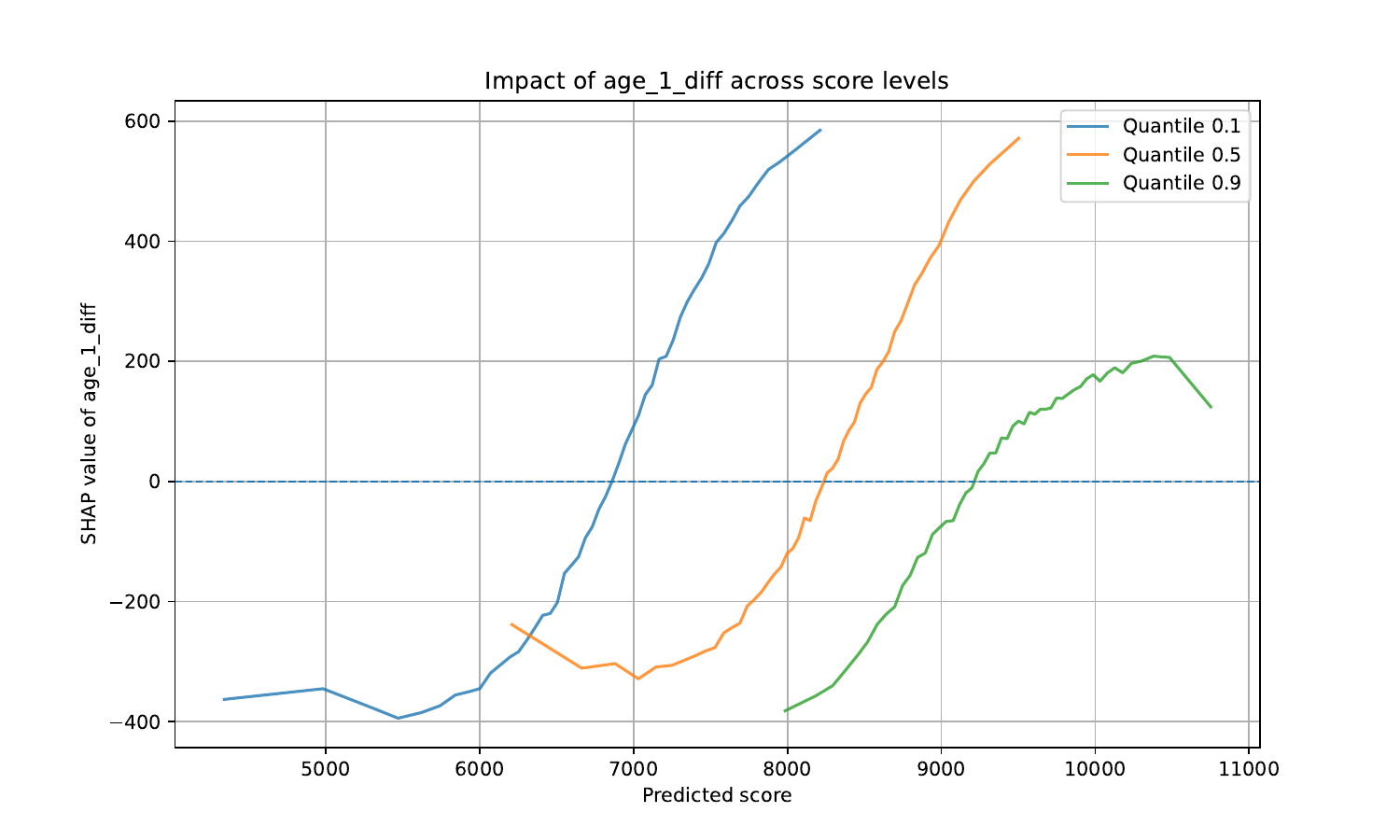}
        \caption{Impact of age difference on predicted quantiles. Larger age differences consistently reduce predicted similarity, with stronger effects at low quantiles.}
        \label{fig:shap_age}
    \end{minipage}\hfill
    \begin{minipage}{0.30\linewidth}
        \centering
        \includegraphics[width=\linewidth]{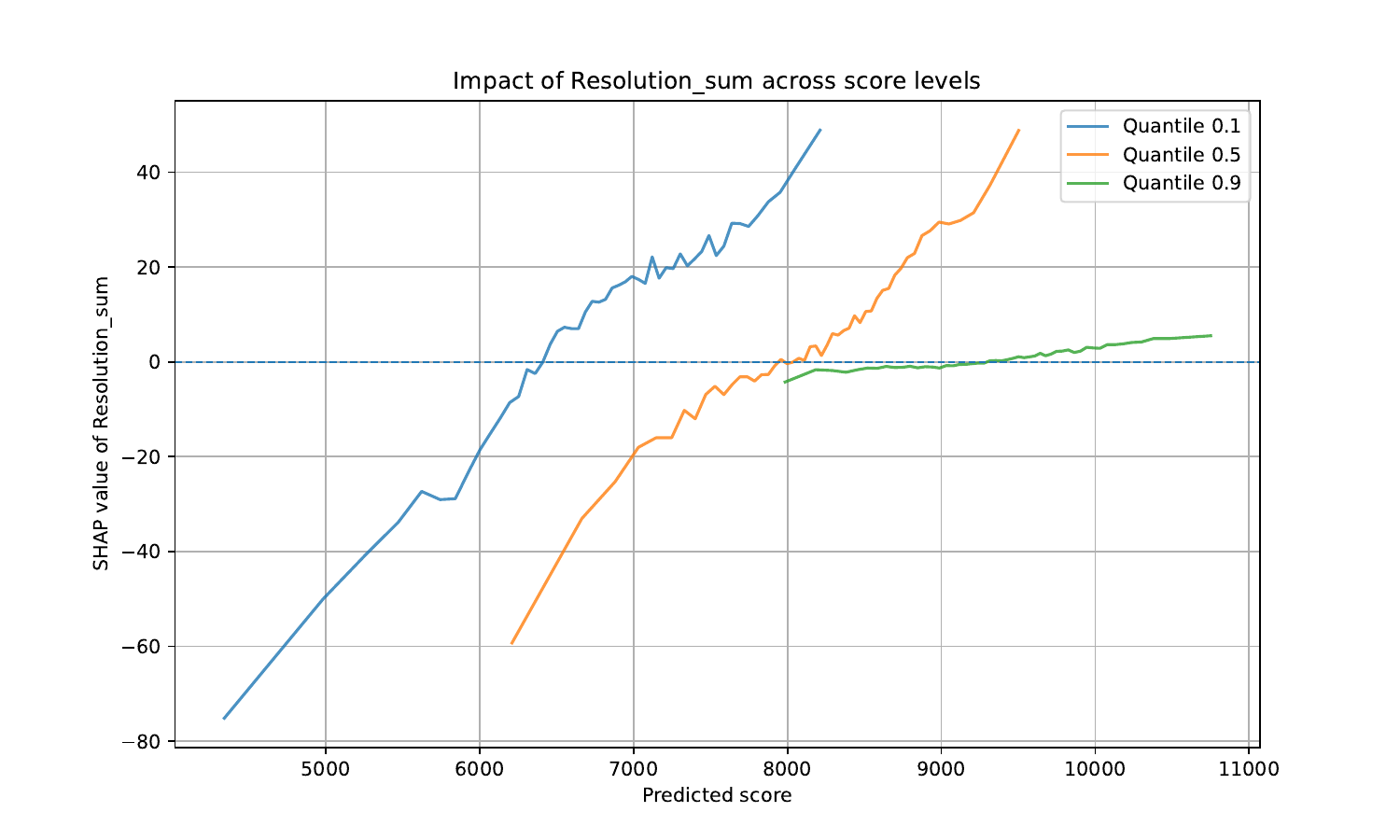}
        \caption{Impact of image resolution sum on predicted quantiles. Higher resolution increases predicted similarity, especially at low to median quantiles.}
        \label{fig:shap_resolution}
    \end{minipage}\hfill
    \begin{minipage}{0.30\linewidth}
        \centering
        \includegraphics[width=\linewidth]{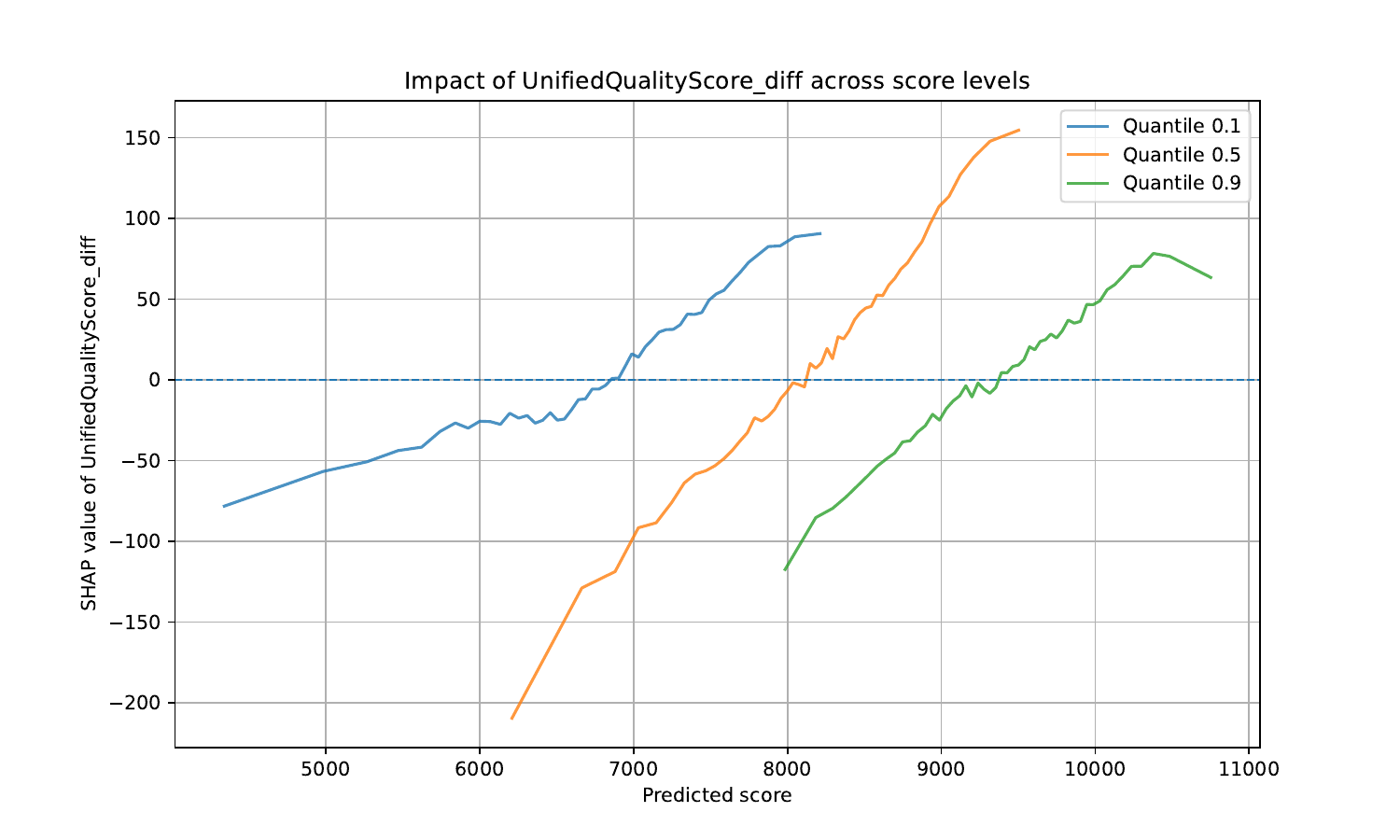}
        \caption{Impact of unified quality score difference on predicted quantiles. Larger differences reduce predicted similarity, with stronger impact at extreme quantiles.}
        \label{fig:shap_quality}
    \end{minipage}
\end{figure}

\paragraph{Approach Validation. }To validate our pairwise quantile regression approach, we first check whether the empirical coverage matches the target quantile levels. Figure~\ref{fig:coverage} confirms that our model achieves the expected coverage across $\tau$ values.

\begin{figure}[h]
    \centering
    \includegraphics[width=0.5\linewidth]{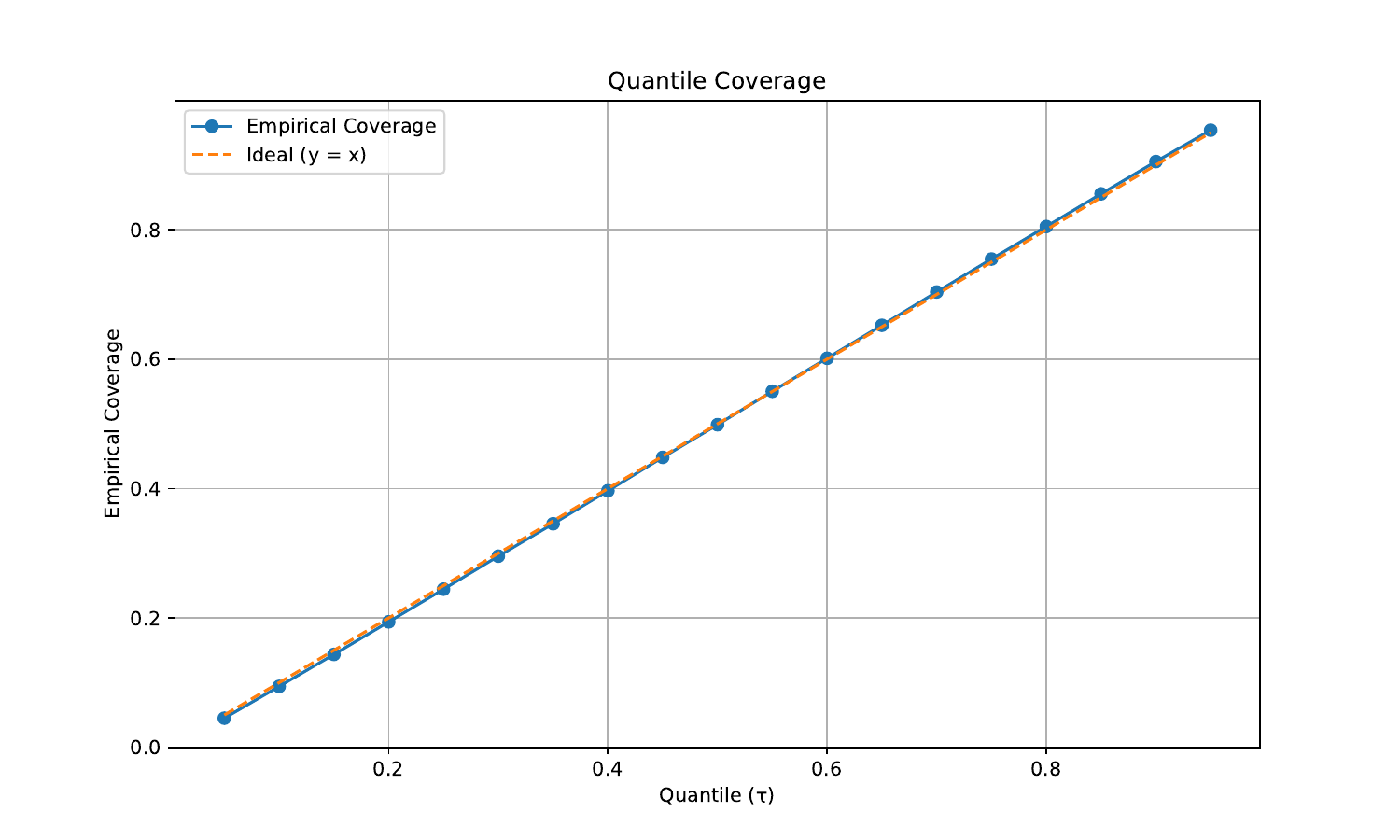}
    \caption{Empirical quantile coverage. The points closely follow the diagonal line, indicating that the predicted quantiles correctly capture the fraction of observed similarity scores below each $\tau$.}
    \label{fig:coverage}
\end{figure}

Next, we examine the relative improvement $D^2_\tau$ across quantiles in Figure~\ref{fig:relative_improvement_pos}. The model shows strong performance across the range of quantiles of interest. The characteristic U-shaped curve is the same as in Figure \ref{fig:Setup-1-Pinball_loss_across_quantiles}. At the quantiles of interest—$(0.01, 0.05)$ for genuine pairs and $(0.95, 0.99)$ for impostor pairs—the model exhibits strong performance.
\begin{figure}[h]
    \centering
    \begin{minipage}[b]{0.45\textwidth}
        \centering
        \includegraphics[width=\textwidth]{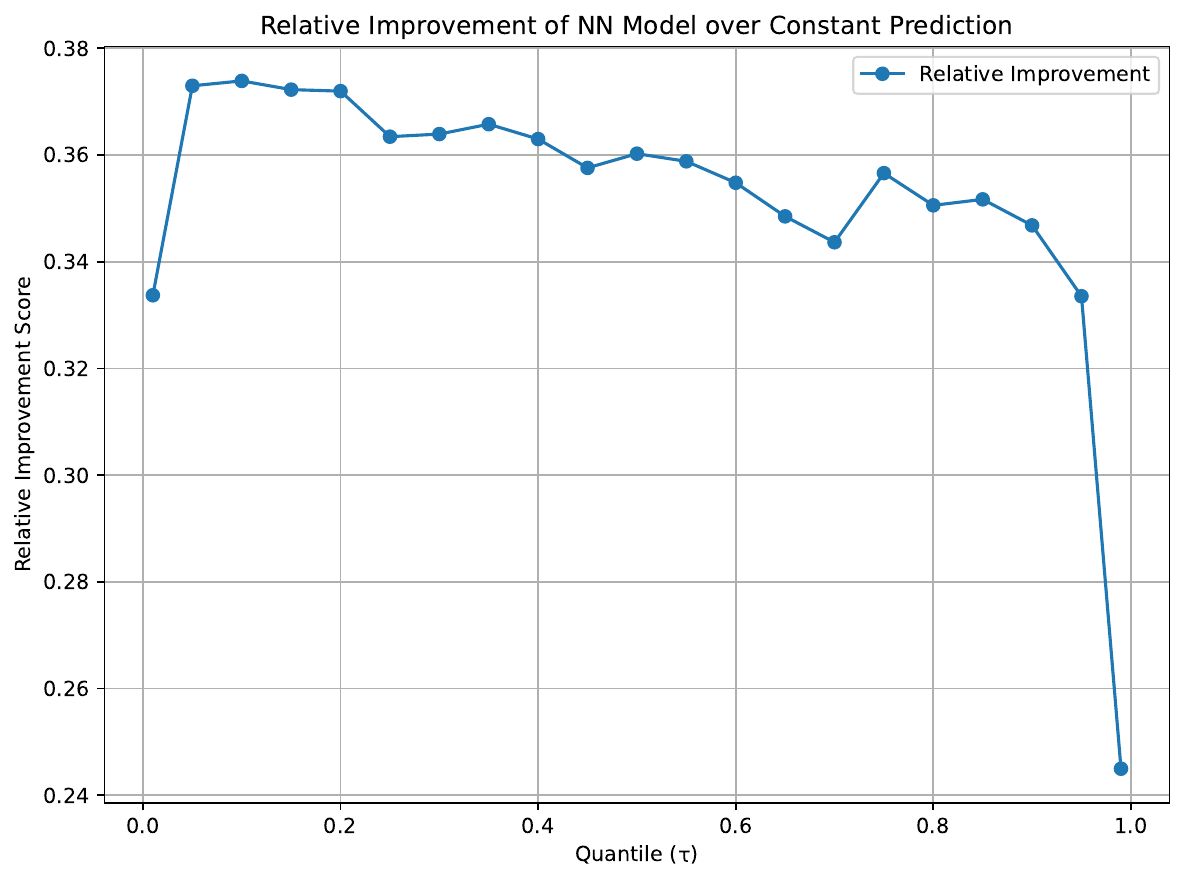}
        \caption{Relative improvement $RI_\tau$ across quantiles for genuine pairs.}
        \label{fig:relative_improvement_pos}
    \end{minipage}
    \hfill
    \begin{minipage}[b]{0.45\textwidth}
        \centering
        \includegraphics[width=\textwidth]{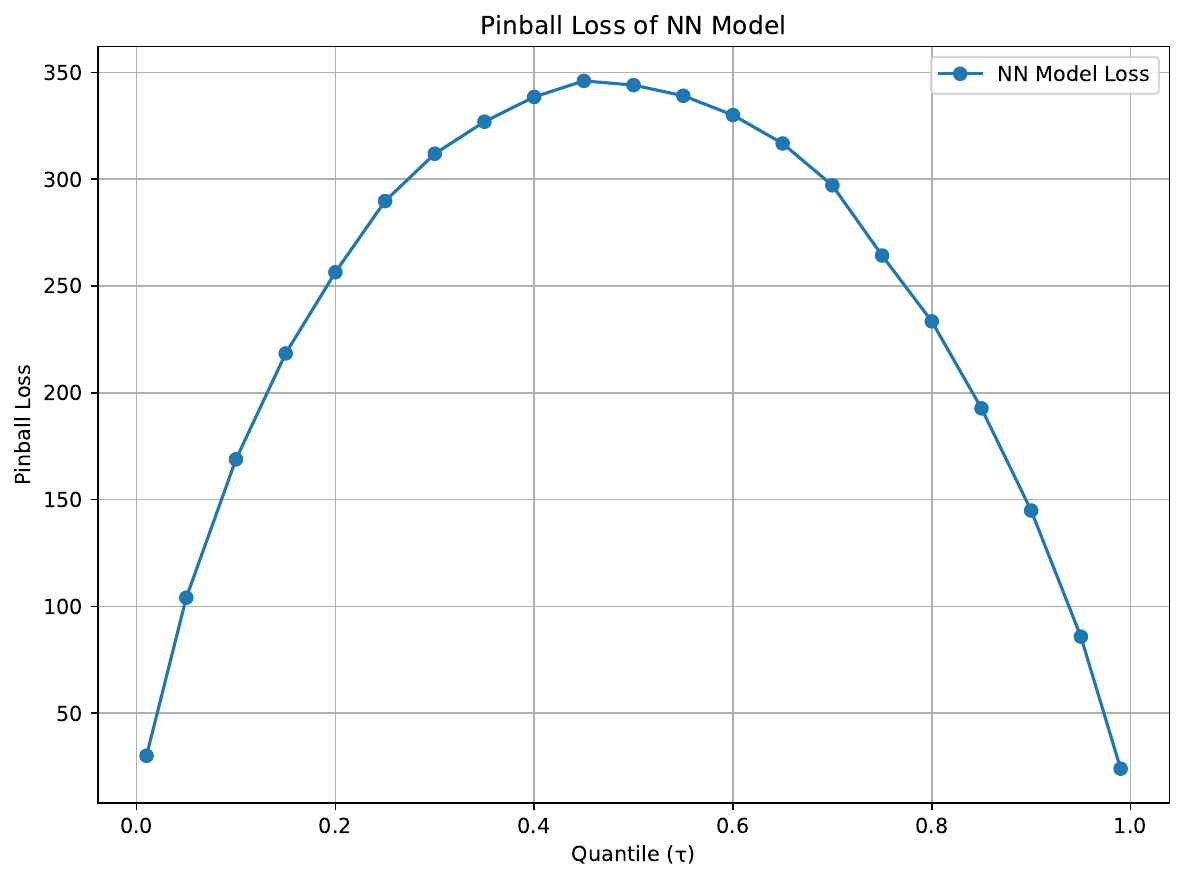}
        \caption{Pinball Loss across quantiles for genuine pairs.}
        \label{fig:pinball_loss_pos}
    \end{minipage}
\end{figure}
\begin{figure}[h]
    \centering
    \begin{minipage}[b]{0.45\textwidth}
        \centering
        \includegraphics[width=\textwidth]{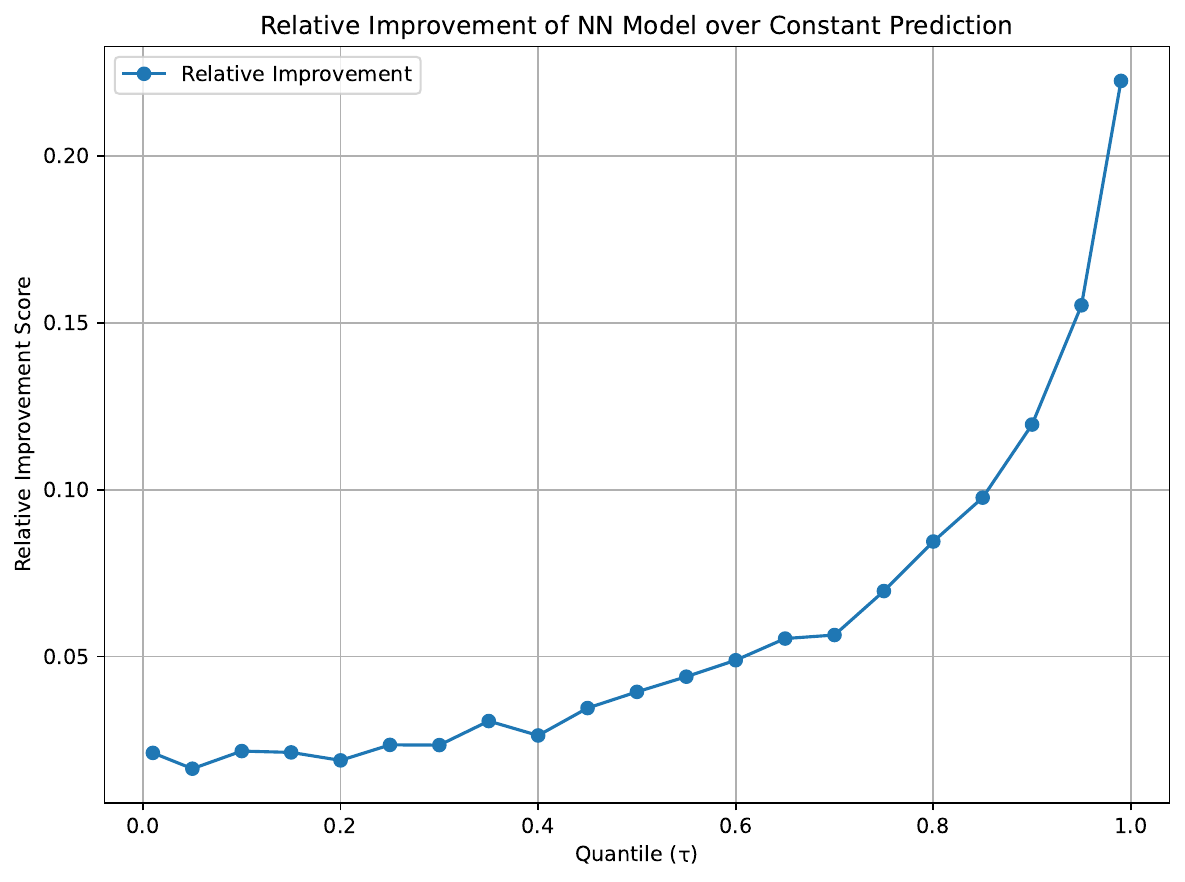}
        \caption{Relative improvement $RI_\tau$ across quantiles for impostor pairs.}
        \label{fig:relative_improvement_neg}
    \end{minipage}
    \hfill
    \begin{minipage}[b]{0.45\textwidth}
        \centering
        \includegraphics[width=\textwidth]{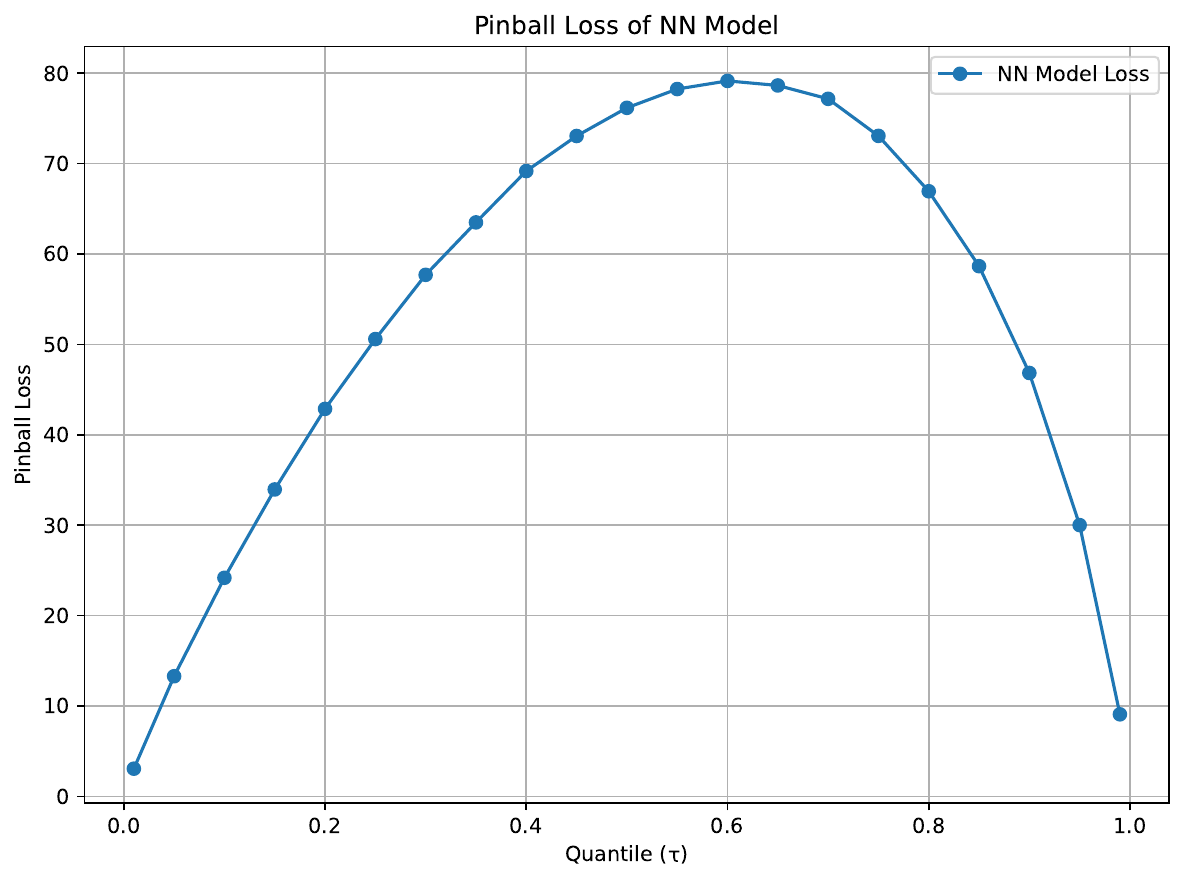}
        \caption{Pinball Loss across quantiles for impostor pairs.}
        \label{fig:pinball_loss_neg}
    \end{minipage}
\end{figure}

\paragraph{Multi-models Feature Validation}
\label{app:feature_validation}
Although the Gradient Boosting and LightGBM models attain lower $D^2$ scores than the neural network (NN), their performance remains sufficient to support informative Shapley value analysis. We assess whether the same features consistently influence predictions across all three models, which would indicate that feature importance is robust to the choice of model architecture.
Table~\ref{tab:d2_scores} presents the $D^2$ scores for both tree-based models across different quantiles. Although these scores are lower than those of the NN (see Table~\ref{tab:fr_results}), they remain within a similar range for both genuine and impostor pairs, suggesting that SHAP-based interpretations are still informative.

\begin{table}[h]
\centering
\begin{subtable}[t]{0.45\textwidth}
\centering
\caption{LightGBM}
\label{tab:d2-LightGBM}
\begin{tabular}{lcc}
\toprule
Quantile Level & Genuine Pairs & Impostor Pairs \\
\midrule
0.01 & 26.2\% & N/A \\
0.05 & 30.0\% & N/A \\
0.95 & N/A & 14.7\% \\
0.99 & N/A & 22.1\% \\
\bottomrule
\end{tabular}
\end{subtable}%
\hfill
\begin{subtable}[t]{0.45\textwidth}
\centering
\caption{Gradient Boosting}
\label{tab:d2-Gradient-Boosting}
\begin{tabular}{lcc}
\toprule
Quantile Level & Genuine Pairs & Impostor Pairs \\
\midrule
0.01 & 25.2\% & N/A \\
0.05 & 31.3\% & N/A \\
0.95 & N/A & 14.4\% \\
0.99 & N/A & 20.0\% \\
\bottomrule
\end{tabular}
\end{subtable}
\caption{$D^2$-score for facial recognition pairs for two models. N/A indicates quantiles not relevant for the given pair type.}
\label{tab:d2_scores}
\end{table}
Thus we can observe that the $D^2$-score, while lower than that of the NN \ref{tab:fr_results}, remains reasonable (above $25\%$). Thus the Shapley Values on such models remain relevant. 
We then compute mean absolute Shapley values for each feature to assess feature importance across models for genuine pairs. Figure~\ref{fig:pos_feature_ranks_heatmap} shows a heatmap of feature ranks at the $0.05$ and $0.01$ quantiles, where lower ranks indicate higher importance. Features such as the age difference and the sum of the image quality are consistently highly ranked across all models, suggesting they are universally influential, while other features exhibit more model-specific reliance. From a practical FR perspective, this indicates that controlling for image quality and age differences is likely to have a broad impact on verification performance across different model types, and that focusing on these key factors can improve robustness and fairness in deployed systems.

\begin{figure}[htbp]
    \centering
    \begin{subfigure}[b]{0.48\linewidth}
        \centering
        \includegraphics[width=\linewidth]{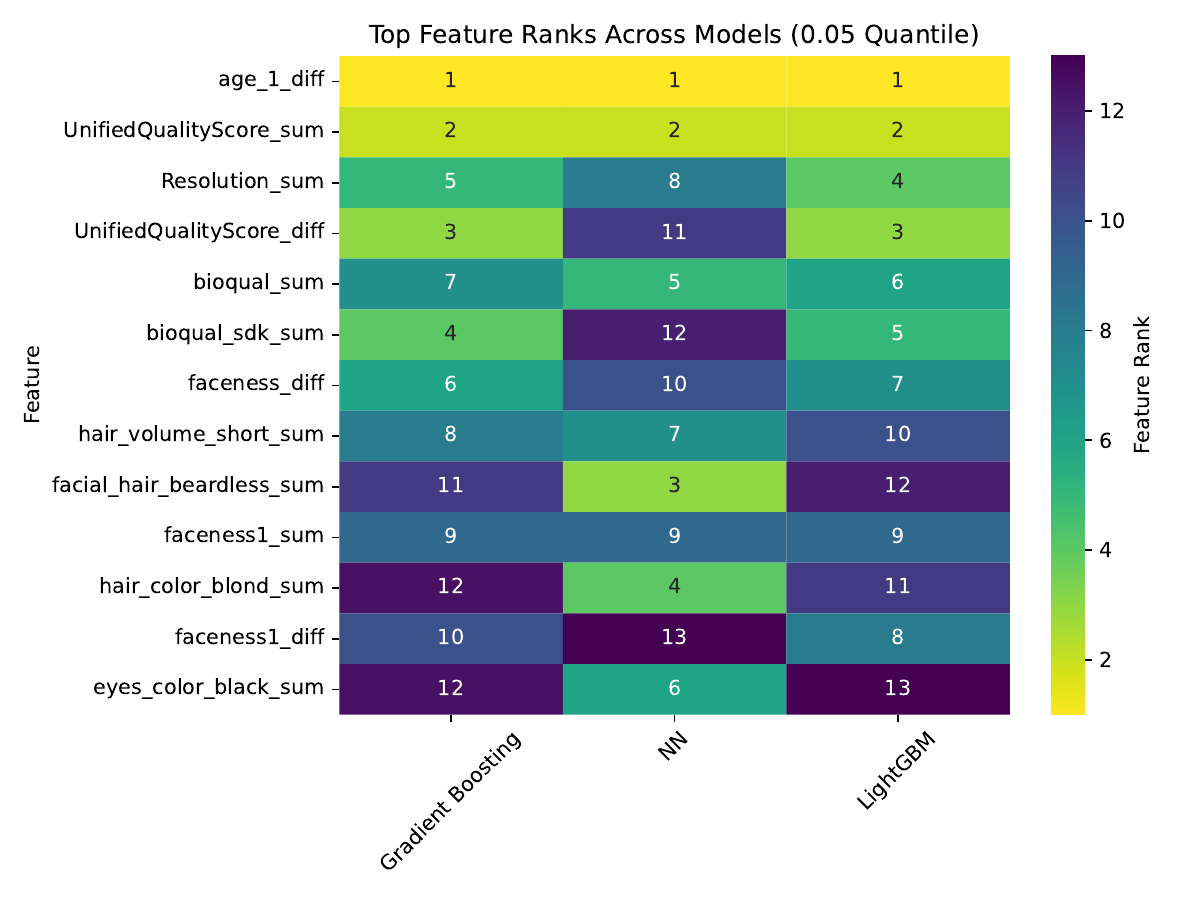}
        \caption{0.05 quantile.}
        \label{fig:pos_05_feature_ranks_heatmap}
    \end{subfigure}
    \hfill
    \begin{subfigure}[b]{0.48\linewidth}
        \centering
        \includegraphics[width=\linewidth]{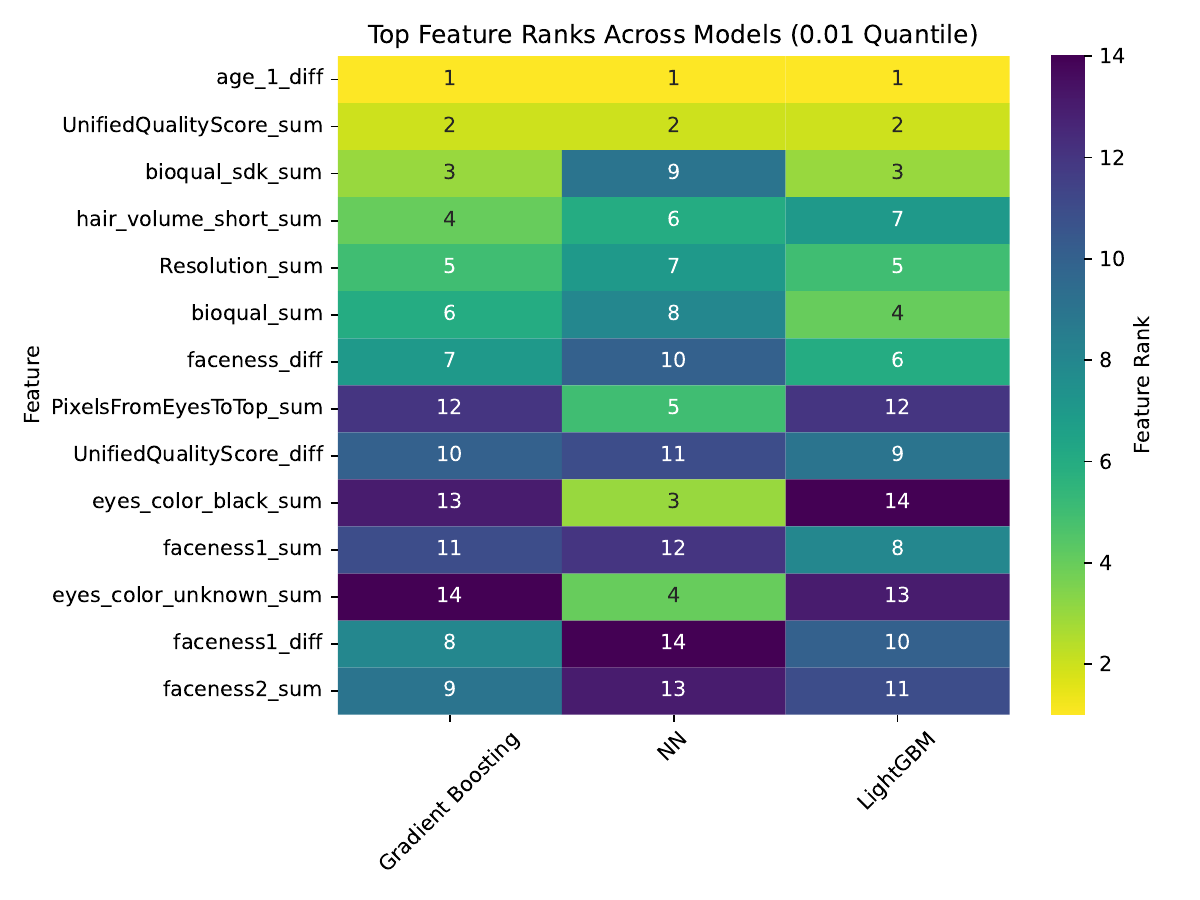}
        \caption{0.01 quantile.}
        \label{fig:pos_01_feature_ranks_heatmap}
    \end{subfigure}
    \caption{Heatmap of feature ranks across three models at quantile $0.05$ and $0.01$ for genuine pairs . Lower ranks indicate higher importance.}
    \label{fig:pos_feature_ranks_heatmap}
\end{figure}
To quantify agreement between models, we calculate Spearman rank correlations between mean absolute SHAP vectors at different quantiles. Table~\ref{tab:pos-spearman-correlation} shows that LightGBM and Gradient Boosting exhibit moderate to high correlation ($0.58$–$0.83$), which is expected given that both are tree-based ensemble models trained on the same features. In contrast, correlations between the NN and tree-based models are lower ($0.18$–$0.24$), reflecting the NN’s distinct feature utilization patterns. Nevertheless, the general trends support consistent identification of key predictive features.

\begin{table}[h]
\centering
\begin{subtable}[t]{0.45\textwidth}
\centering
\caption{Quantile $0.01$}
\begin{tabular}{lccc}
\toprule
 & NN & LightGBM & Gradient Boosting \\
\midrule
NN & 1 & 0.21 & 0.18 \\
LightGBM & 0.21 & 1 & 0.58 \\
Gradient Boosting & 0.18 & 0.58 & 1 \\
\bottomrule
\end{tabular}
\end{subtable}
\hfill
\begin{subtable}[t]{0.45\textwidth}
\centering
\caption{Quantile $0.05$}
\begin{tabular}{lccc}
\toprule
 & NN & LightGBM & Gradient Boosting \\
\midrule
NN & 1 & 0.24 & 0.23 \\
LightGBM & 0.24 & 1 & 0.83 \\
Gradient Boosting & 0.23 & 0.83 & 1 \\
\bottomrule
\end{tabular}
\end{subtable}
\caption{Spearman rank correlation between mean absolute SHAP vectors across models at different operating quantiles for genuine pairs.}
\label{tab:pos-spearman-correlation}
\end{table}

\begin{table}[h]
\centering
\begin{subtable}[t]{0.45\textwidth}
\centering
\caption{Quantile $0.95$}
\begin{tabular}{lccc}
\toprule
 & NN & LightGBM & Gradient Boosting \\
\midrule
NN & 1 & 0.36 & 0.31 \\
LightGBM & 0.36 & 1 & 0.84 \\
Gradient Boosting & 0.31 & 0.84 & 1 \\
\bottomrule
\end{tabular}
\end{subtable}
\hfill
\begin{subtable}[t]{0.45\textwidth}
\centering
\caption{Quantile $0.99$}
\begin{tabular}{lccc}
\toprule
 & NN & LightGBM & Gradient Boosting \\
\midrule
NN & 1 & 0.32 & 0.21 \\
LightGBM & 0.32 & 1 & 0.72 \\
Gradient Boosting & 0.21 & 0.72 & 1 \\
\bottomrule
\end{tabular}
\end{subtable}
\caption{Spearman rank correlation between mean absolute SHAP vectors across models at different operating quantiles for impostor pairs.}
\label{tab:neg-spearman-correlation}
\end{table}

\section{Features Explanation}
\begin{table}[H]
\centering
\begin{tabular}{|
                >{\raggedright\arraybackslash}p{0.25\textwidth} 
                >{\raggedright\arraybackslash}p{0.25\textwidth} 
                |
                >{\raggedright\arraybackslash}p{0.25\textwidth} 
                >{\raggedright\arraybackslash}p{0.25\textwidth}
                |}
\toprule
\textbf{Feature} & \textbf{Description} & \textbf{Feature} & \textbf{Description} \\
\midrule
EyeGlassesPresent & Indicates if the subject is wearing glasses. &
SunGlassesPresent & Indicates if the subject is wearing sunglasses. \\

Underexposure & Image underexposure measure. &
Overexposure & Image overexposure measure. \\

BackgroundUniformity & Uniformity of the image background. &
FaceOcclusion & Presence of occlusions on the face. \\

Resolution & Image resolution. &
InterEyeDistance & Distance between the eyes. \\

MotionBlur & Amount of motion blur in the image. &
CompressionArtifacts & Level of compression artifacts. \\

PixelsFromEyeToLeftEdge & Distance from eyes to the left image edge. &
PixelsFromEyeToRightEdge & Distance from eyes to the right image edge. \\

PixelsFromEyesToBottom & Distance from eyes to bottom of face region. &
PixelsFromEyesToTop & Distance from eyes to top of face region. \\

UnifiedQualityScore & Combined image quality metric. &
BLeft & Bounding box left coordinate. \\

BTop & Bounding box top coordinate. &
BWidth & Bounding box width. \\

BHeight & Bounding box height. &
bioqual & Biometric quality score. \\

yaw & Horizontal head rotation. &
pitch & Vertical head rotation. \\

roll & Head tilt angle. &
compression & Compression ratio of the image. \\

eye\_opening\_px & Eye opening size in pixels. &
mouthopening\_px & Mouth opening size in pixels. \\

glasses & Binary indicator for glasses presence. &
wear\_sunglasses & Binary indicator for sunglasses presence. \\

age\_1 & First age estimate. &
no\_beard & Absence of facial hair. \\

full\_beard & Presence of full beard. &
mustache & Presence of mustache. \\

faceness1 & Confidence score of face detection. &
faceness2 & Secondary confidence score for face detection. \\

illumination\_left & Illumination on left side of face. &
illumination\_right & Illumination on right side of face.\\

wear\_covid\_mask & Presence of a face mask.&
age\_2 & Second age estimate. \\

eyes\_color\_unknown & Eye color unknown. &
eyes\_color\_green & Green eyes. \\

eyes\_color\_blue & Blue eyes. &
eyes\_color\_gray & Gray eyes. \\

eyes\_color\_black & Black eyes. &
eyes\_color\_brown & Brown eyes. \\

eyes\_color\_heterochromic & Heterochromia in eyes. &
face\_shape\_unknown & Unknown face shape. \\

face\_shape\_square & Square-shaped face. &
face\_shape\_oval & Oval-shaped face. \\

face\_shape\_oblong & Oblong-shaped face. &
face\_shape\_round & Round face. \\

face\_shape\_triangular & Triangular face. &
hair\_color\_unknown & Unknown hair color. \\

hair\_color\_white & White hair. &
hair\_color\_gray & Gray hair. \\

hair\_color\_salted & Salt-and-pepper hair. &
hair\_color\_black & Black hair. \\

hair\_color\_brown & Brown hair. &
hair\_color\_red & Red hair. \\

hair\_color\_blond & Blond hair. &
hair\_volume\_unknown & Unknown hair length/volume. \\

hair\_volume\_long & Long hair. &
hair\_volume\_medium & Medium hair length. \\

hair\_volume\_short & Short hair. &
hair\_volume\_partshaved & Partially shaved hair. \\

hair\_volume\_shaved & Shaved hair. &
facial\_hair\_unknown & Unknown facial hair. \\

facial\_hair\_beard & Beard present. &
facial\_hair\_goatee & Goatee present. \\

facial\_hair\_goatee\_and\\\_mustache & Goatee with mustache. &
facial\_hair\_chinstrap & Chinstrap beard. \\

facial\_hair\_mustache & Mustache only. &
facial\_hair\_sideburns & Sideburns present. \\

facial\_hair\_beardless & No facial hair. & & \\

\bottomrule
\end{tabular}
\caption{List of features with brief descriptions.}
\label{tab:features}
\end{table}
\end{document}